\documentclass[11pt]{article}
\usepackage[margin=1in]{geometry}
\usepackage{amsmath,amssymb,amsthm,mathtools}
\usepackage{booktabs,xcolor,enumitem,array}
\usepackage[colorlinks=true,linkcolor=blue!50!black,citecolor=blue!50!black,urlcolor=blue!50!black]{hyperref}
\usepackage{tikz}
\usepackage{pgfplots}
\pgfplotsset{compat=1.16}
\usepgfplotslibrary{groupplots}
\usepackage{algorithm}
\usepackage{algpseudocode}
\usetikzlibrary{arrows.meta,positioning,fit,calc}
\algrenewcommand\algorithmicrequire{\textbf{Input:}}
\algrenewcommand\algorithmicensure{\textbf{Output:}}

\newtheorem{theorem}{Theorem}
\newtheorem{lemma}{Lemma}
\newtheorem{proposition}{Proposition}
\newtheorem{corollary}{Corollary}

\theoremstyle{definition}
\newtheorem{definition}{Definition}
\newtheorem{example}{Example}
\theoremstyle{remark}
\newtheorem{remark}{Remark}

\newcommand{\E}{\mathbb{E}}
\newcommand{\I}{I}
\newcommand{\KL}{\mathrm{KL}}
\newcommand{\Jl}{\mathcal J_\lambda}
\newcommand{\Js}{J^\star}
\newcommand{\JF}{J^{F}}
\newcommand{\Gap}{\Delta}
\newcommand{\epsF}{\epsilon_F}
\newcommand{\R}{\mathbb R}
\newcommand{\cX}{\mathcal X}
\newcommand{\cY}{\mathcal Y}

\newcommand{\cC}{\mathcal C}

\newcommand{\cD}{\mathcal D}

\newcommand{\Pset}{\mathcal P(\cY)}
\newcommand{\Ub}{\mathcal U}
\newcommand{\TV}{\mathrm{TV}}
\newcommand{\hb}{h_{\rm b}}
\newcommand{\Id}{\mathbf I}
\newcommand{\Aw}{\tilde A}
\newcommand{\Ft}{\tilde F}
\newcommand{\Sg}{S_{\mathrm G}}

\DeclareMathOperator*{\argmin}{arg\,min}
\newcolumntype{L}[1]{>{\raggedright\arraybackslash}p{#1}}

\title{What Should an Information Bottleneck Encoder Learn?\\ Separating Representation Learning from Rate Control}

\author{
Hassan Tavakoli, \emph{Member, IEEE}\\\
{School of EECS}\\
{Oregon State University}\\
Corvallis, OR 97331, USA\\
\texttt{tavakolh@oregonstate.edu}
\and
Bella Bose, \emph{Life Fellow, IEEE}\\
{School of EECS}\\
{Oregon State University}\\
Corvallis, OR 97331, USA\\
\texttt{Bella.Bose@oregonstate.edu}
}

\date{}

\begin{document}
\maketitle

\begin{abstract}
The information bottleneck (IB) trades compression of an input $X$ against prediction of a target $Y$ through a multiplier $\beta=1/\lambda$, yet each operating point is usually served by training, or fully conditioning, a new encoder.
For jointly Gaussian data, we show that the IB solution factorizes as $A(\beta)=D(\beta)W$: the information-bearing directions $W$ do not depend on $\beta$, and $\beta$ acts only through scalar gains.
We ask whether this separation of \emph{representation} from \emph{rate control} survives beyond Gaussian models, treating the budget as a given operating parameter rather than a quantity to be optimized.
For a budget-free representation \(F(X)\), the \emph{separation gap} is controlled by the \emph{sufficiency deficit} $\epsilon_F=I(X;Y\mid F(X))$: We show there is at least $\epsilon_F$ in the worst case for finite inputs, and it is bounded uniformly in $\lambda$ by a function of $\epsilon_F$ for finite label alphabets. The same certificate extends to a new target distribution and to pooled multi-domain data, where a conditional-alignment term appears.
On MNIST, SVHN, CIFAR-10 and STL-10 (10 seeds, eight prespecified budgets), a frozen CNN representation with one compressor--decoder pair per budget reduces estimated training multiply--accumulate operations (MACs) by $7.7\times$ and wall-clock time by about $2\times$ relative to end-to-end VIBs, and attains higher budget-averaged accuracy ($65.6\%$ versus $57.3\%$ per domain; $67.1\%$ versus $57.1\%$ pooled), with the advantage concentrated at strong compression, where the VIBs collapse. At weak compression ($\lambda\le0.036$), VIB has higher cross-entropy and ECE; its surrogate Lagrangian is higher per domain, indistinguishable from the pooled VIB in the macro average, and lower on the pooled test set.
\end{abstract}


%% ---------------------------------------------------------------------------
\section{Introduction}\label{sec:intro}
The IB of Tishby, Pereira and Bialek \cite{tishby99} seeks a stochastic representation $T$ of $X$ that is maximally informative about $Y$ subject to a rate constraint on $I(X;T)$. The Lagrangian turns $\beta$ into a control knob whose solution path has phase transitions \cite{chechik05,wu19}. Deep variational IB (VIB) \cite{alemi17} makes the problem tractable for neural encoders, but a trained network realizes one $\beta$. Rate-conditioned networks (FiLM \cite{perez18}; multi-rate and modulated autoencoders \cite{bae23,yang20}) condition the \emph{whole} network on $\beta$. We ask the information-theoretic question left open: \emph{what part of the IB solution must know $\beta$?} The paper concerns VIB methods for representation learning and adaptation, with the budget $\beta$ a given input that changes at deployment and is never optimized here.

\paragraph{The Gaussian clue.}
For jointly Gaussian $(X,Y)$ the IB optimum is $T_\beta=A(\beta)X+\xi$ with
\begin{equation}
A(\beta)=D(\beta)\,W,\qquad \alpha_i(\beta)^2=\frac{[\beta\rho_i^2-1]_+}{1-\rho_i^2},
\label{eq:intro-factor}
\end{equation}
where $W$ is a budget-free linear map computed from second moments and $D(\beta)=\mathrm{diag}(\alpha_i(\beta))$ carries all $\beta$-dependence. Changing the budget changes the \emph{rate allocation}, not the \emph{information-bearing coordinates}. This is the known Gaussian IB solution \cite{chechik05}; it is our motivation.

\paragraph{Representation--rate separation.}
We study the architecture (Fig.~\ref{fig:arch})
\begin{equation}
U=F(X),\qquad T\sim C_\lambda(\cdot\mid U),\qquad \widehat Y=G_\lambda(T),\qquad \lambda=1/\beta,
\label{eq:arch}
\end{equation}
where $F$ is learned without a rate penalty and then frozen, while the compressor $C_\lambda$ and decoder $G_\lambda$ are fitted to the desired operating point. In the Gaussian case $F(X)=WX$ and $C_\beta(u)=D(\beta)u+\xi$. The decoder is the posterior map of $T_\beta$ and does depend on $\beta$; we claim no budget-free decoder.

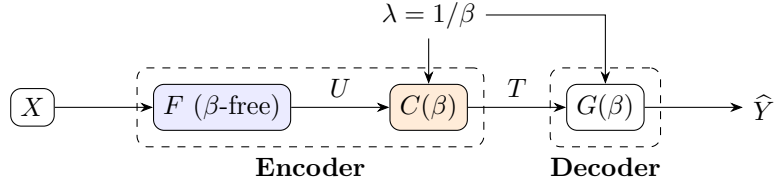
\begin{figure}[t]
\centering
\begin{tikzpicture}[node distance=13mm,>=Stealth,every node/.style={font=\small}]
\node[draw,rounded corners] (x) {$X$};
\node[draw,rounded corners,right=of x,fill=blue!8] (f) {$F$ (\(\beta\)-free)};
\node[draw,rounded corners,right=of f,fill=orange!15] (c) {$C (\beta)$};
\node[draw,rounded corners,right=of c] (g) {$G (\beta)$};
\node[right=of g] (y) {$\widehat Y$};
\node[above=6mm of c] (b) {$\lambda=1/\beta$};
\draw[->] (x)--(f);\draw[->] (f)--node[above]{$U$}(c);\draw[->] (c)--node[above]{$T$}(g);\draw[->] (g)--(y);\draw[->] (b)--(c);\draw[->] (b)-|(g);
% two dashed boxes
\node[draw,dashed,rounded corners,inner sep=6pt,fit=(f)(c),
      label={[font=\small\bfseries]below:Encoder}] (enc) {};
\node[draw,dashed,rounded corners,inner sep=6pt,fit=(g),
      label={[font=\small\bfseries]below:Decoder}] (dec) {};
\end{tikzpicture}
\caption{Representation--rate separation. $F$ is trained once without reference to the budget and characterized by the scalar $\epsilon_F\triangleq I(X;Y|F(X))$; \(C\) and \(G\) depend on the budget, \(\beta\) .}
\label{fig:arch}
\end{figure}

\paragraph{Questions and contributions.}
\textbf{(Q1) Separation:} when is a $\beta$-free $U=F(X)$ sufficient for the entire IB path, how does the loss grow with its sufficiency deficit, and what must the $\beta$-dependent module express? \textbf{(Q2) Cost:} what are the accuracy, surrogate-objective and compute consequences of freezing $F$ in practice?
\begin{enumerate}[leftmargin=*,itemsep=1pt]
\item \textbf{Separation theory (Q1; theorems).} An exact decomposition of the separation gap into representation and compressor terms; exact separation iff $\epsF=0$; the uniform bound; certificates for adaptation and pooled domains; an exact decoder-error decomposition and a decoder-sharing criterion for pooled domains; exact separation for Gaussians with a characterization of sufficient linear representations.
\item \textbf{Evidence (Q2; numerical).} A 10-seed study on four datasets with two separated settings (per-domain, pooled) and six baseline settings (independent VIB, FiLM-conditioned VIB, OIB; per-domain and pooled), confidence intervals, and an explicit cost model.
\end{enumerate}

\begin{table}[t]
\centering\small
\caption{Positioning. We treat $\beta$ as a \emph{given} operating parameter and ask
what must depend on it.}
\label{tab:position}
\begin{tabular}{@{}L{4.5cm}L{2cm}L{5cm}@{}}\toprule
Method & Role of $\beta$ & What depends on $\beta$ \\\midrule
VIB, nonlinear IB \cite{alemi17,kolchinsky19} & given & A whole Encoder/Decoder networks (one per $\beta$)\\\midrule
Rate-condition \cite{perez18,bae23,yang20} & given & A whole network \\\midrule
Flexible VIB (FVIB) \cite{kudo24} & given & closed-form rescaling (Taylor approx.) \\\midrule
Opportunistic IB (OIB) \cite{binucci24} & given & Gaussian-IB linear loadings \\\midrule
$\beta$-free IB \cite{wang26} & -- & $\beta$-free  \\\midrule
\textbf{Ours} & \textbf{given} & New Separation method: only $(C,G)$ per \(\beta\) and $F$ is $\beta$-free  \\
\bottomrule\end{tabular}
\end{table}

%% ---------------------------------------------------------------------------
\section{Problem formulation and related work}\label{sec:prelim}

Logarithms are natural; information is in nats. $X$ takes values in a standard Borel space $\cX$ and $Y$ in a finite set $\cY$, $K=|\cY|$, except in Section~\ref{sec:gaussib} and Theorem~\ref{thm:gauss}, where $Y$ is continuous. A \emph{channel} is a conditional law $P_{T|X}$ with Markov chain $Y-X-T$. For $\lambda>0$,
\begin{align}
\Jl(P_{T|X}) & =H(Y|T)+\lambda\,\I(X;T)=H(Y)-\I(T;Y)+\lambda\,\I(X;T), &
\Js(\lambda) & =\inf_{P_{T|X}}\Jl(P_{T|X}).
\label{eq:J}
\end{align}
For continuous $Y$ we use $\Jl:=-\I(T;Y)+\lambda\I(X;T)$, which differs by the constant $H(Y)$; all gaps are unaffected. Multiplying \eqref{eq:J} by $\beta$ recovers $\mathcal J_\beta=I(X;T)-\beta I(T;Y)$ up to a constant. $\Js$ is an infimum of affine functions of $\lambda$, hence concave and nondecreasing, with $H(Y|X)\le\Js(\lambda)\le H(Y)$. The relevance--rate function is $R^\star(\mathsf r)=\sup\{\I(T;Y):\I(X;T)\le\mathsf r\}$.

\paragraph{Variational surrogate.}
Neural IB minimizes $\mathcal L_\lambda=\mathrm{CE}+\lambda R$ with $\mathrm{CE}=\E[-\log g(Y|T)]$ for a decoder $g$ and $R=\E_X\KL(P_{T|X}\|\nu)$ for a prior $\nu$ (equivalently $\mathrm{CE}+\tfrac1\beta D_{\rm KL}(q(T|X)\|\nu)$).
\begin{lemma}[Variational upper bound]\label{lem:var}
For every decoder $g$ and prior $\nu$, $\mathcal L_\lambda\ge\Jl(P_{T|X})$, with equality iff $g=P_{Y|T}$ and $\nu=P_T$.
\end{lemma}
\begin{proof}
$\E[-\log g(Y|T)]=H(Y|T)+\E\,\KL(P_{Y|T}\|g)$ and $\E_X\KL(P_{T|X}\|\nu)=\I(X;T)+\KL(P_T\|\nu)$.
\end{proof}
For each operating point, the budget $\lambda$ is specified externally and is not optimized by our method. The theoretical results hold for each fixed $\lambda$, and bounds uniform in $\lambda$ provide worst-case guarantees. In the experiments, the operating grid is prespecified and shared by all methods. The separated pipeline trains one compressor--decoder pair for each grid point, whereas each VIB baseline is trained independently at each grid point. No operating point is selected using test data.

\paragraph{Related work.}
\emph{IB theory:} the IB \cite{tishby99}, its Gaussian solution \cite{chechik05}, phase transitions and learnability \cite{wu19}, generalization \cite{shamir10}, and Lagrangian-versus-constraint issues \cite{rodriguez20}. \emph{Variational and conditional IB:} VIB \cite{alemi17} and nonlinear IB \cite{kolchinsky19} train one encoder per $\beta$; rate-conditioned networks \cite{bae23,yang20,perez18} condition the whole network on the rate. FVIB \cite{kudo24} trains a network without $\beta$ and lets $\beta$ enter through a closed-form rescaling, with a guarantee for a Taylor-approximated classification objective; opportunistic IB \cite{binucci24} places Gaussian-IB loadings on a fixed linear map and a pretrained network; $\beta$-free IB \cite{wang26} addresses learning or tuning $\beta$ itself; optimizing the budget is outside our scope, since we treat it as given.
%% ---------------------------------------------------------------------------
\section{Gaussian IB: the factorization that motivates the paper}\label{sec:gaussib}
Let $(X,Y)$ be jointly Gaussian, $X\in\R^{d_x}$, $Y\in\R^{d_y}$, $\E X=0$, $\E Y=0$ (w.l.o.g., by translation invariance of mutual information), $\Sigma_x\succ0$, $\Sigma_{x|y}=\Sigma_x-\Sigma_{xy}\Sigma_y^{-1}\Sigma_{yx}$, with canonical correlations $1>\rho_1\ge\dots\ge\rho_r>0$ (singular values of $\Sigma_x^{-1/2}\Sigma_{xy}\Sigma_y^{-1/2}$, $r\le\min(d_x,d_y)$). The numbers $\kappa_i=1-\rho_i^2$ are the eigenvalues of $M=\Sigma_x^{-1/2}\Sigma_{x|y}\Sigma_x^{-1/2}=Q\,\mathrm{diag}(\kappa)\,Q^\top$ (ascending, padded with $\kappa_i=1$ beyond $r$). The IB optimum over \emph{all} channels is attained by a linear--Gaussian channel \cite{chechik05}, $T=AX+\xi$, $\xi\sim\mathcal N(0,\Id)$ w.l.o.g.,  any $\Sigma_\xi\succ0$ is absorbed into $A$ by an invertible map of $T$. All Gaussian representations are denoted $\Sg$.

\begin{proposition}[Gaussian IB, our notation]\label{prop:gib}
The optimal channel for $\mathcal J_\beta=I(X;T)-\beta I(T;Y)$ is $T=A(\beta)X+\xi$,
\begin{equation}
A(\beta)=D(\beta)\,Q^\top\Sigma_x^{-1/2},\qquad
\alpha_i(\beta)^2=\frac{[\beta\rho_i^2-1]_+}{1-\rho_i^2},\qquad \beta_i^{c}=\frac1{\rho_i^2}.
\label{eq:gib-sol}
\end{equation}
Mode $i$ is active iff $\beta>\beta_i^c$, and $I(X;T)=\tfrac12\sum_i\log(1+\alpha_i(\beta)^2)$. In the normalization of \cite{chechik05}, with $\{\mathbf v_i\}$ left eigenvectors of $\Sigma_{x|y}\Sigma_x^{-1}$, $\beta_i^c=\frac1{1-\kappa_i}$ and $\alpha_i^2=\frac{\beta(1-\kappa_i)-1}{\kappa_i\,r_i}$, $r_i=\mathbf v_i^\top\Sigma_x\mathbf v_i$; our $\mathbf v_i=\Sigma_x^{-1/2}q_i$ has $r_i=1$.
\end{proposition}
\begin{proof}
Global optimality over all channels is \cite{chechik05}. With $T|X\sim\mathcal N(AX,\Id)$ and 
\[
T|Y\sim\mathcal N(A\E[X|Y],A\Sigma_{x|y}A^\top+\Id), 
\]
\(I(X;T)=\tfrac12\log\det(A\Sigma_xA^\top+\Id),
I(T;Y)=\tfrac12\log\det(A\Sigma_xA^\top+\Id)-\tfrac12\log\det(A\Sigma_{x|y}A^\top+\Id),\) 
give
\[
I(X;T)-\beta I(T;Y)=\tfrac12(1-\beta)\log\det(A\Sigma_xA^\top+\Id)+\tfrac{\beta}{2}\log\det(A\Sigma_{x|y}A^\top+\Id).
\]
Writing $A=\Aw\,\Sigma_x^{-1/2}$ and defining $\mathcal L_\beta(\Aw):=I(X;T)-\beta I(T;Y)$ for the channel induced by $A$, we get
\[
\mathcal L_\beta(\Aw)=\tfrac12(1-\beta)\log\det(\Aw\Aw^\top+\Id)+\tfrac\beta2\log\det(\Aw M\Aw^\top+\Id),
\]
and for $\Aw=DQ^\top$ it separates
as $\sum_if_i(a_i)$, $a_i=\alpha_i^2$, $f_i(a)=\tfrac12(1-\beta)\log(1+a)+\tfrac\beta2\log(1+\kappa_ia)$, with
\(
f_i'(a)=\frac{(1-\beta\rho_i^2)+\kappa_i a}{2(1+a)(1+\kappa_ia)}.
\)
If $\beta\rho_i^2\le1$ then $f_i'\ge0$ and $a=0$ is optimal; otherwise $f_i'<0$ for $a<a_i^\star$ and $f_i'>0$ beyond $a_i^\star=(\beta\rho_i^2-1)/\kappa_i$, the unique minimizer. Stationarity of the full matrix problem, 
$(\beta-1)(\Aw\Aw^\top+\Id)^{-1}\Aw=\beta(\Aw M\Aw^\top+\Id)^{-1}\Aw M$, row $i$ at $\Aw=DQ^\top$,
gives the same $a_i^\star$.
\end{proof}

\paragraph{The factorization.}
Let $W:=Q^{\top}\Sigma_x^{-1/2}$ and $\Sg:=WX$. It depends only on second moments, with $W\Sigma_xW^{\top}=\Id$ and $W\Sigma_{x|y}W^{\top}=\Id-\mathrm{diag}(\rho_i^2)$. Then
\(
A(\beta)=D(\beta)\,W,\, T_\beta=D(\beta)\,WX+\;\xi,\, \xi\sim\mathcal N(0,\Id).
\)
$W$ is a budget-free linear map playing the role of $F$; $D(\beta)$ is diagonal with at most $r$ nonzero entries, each a closed-form function of one $\rho_i$ and of $\beta$, playing the role of $C_\beta$. In the coordinates $\Sg$ the problem decouples into scalar channels $T_i=\alpha_i(\beta)(\Sg)_i+\xi_i$.

\begin{proposition}[Sufficient linear representations]\label{prop:gausssuff}
Let $\Sg=LX$ with $L\in\R^{m\times d_x}$ of full row rank and $B_{yx}=\Sigma_{yx}\Sigma_x^{-1}$. Then $\epsilon_F\triangleq I(X;Y|LX)=0$ 
iff $B_{yx}=\Omega L$ for some matrix $\Omega$ (the row space of $L$ contains that of $B_{yx}$), and in general
\[
\epsilon_F=I(X;Y)-I(LX;Y)=\tfrac12\sum_i\big[\log(1-\rho_i(L)^2)-\log(1-\rho_i^2)\big]\ \ge0,
\]
where $\rho_i(L)\le\rho_i$ are the canonical correlations of $(LX,Y)$ (padded with zeros). 
\end{proposition}
\begin{proof}
$Y|X\sim\mathcal N(B_{yx}X,\Sigma_{y|x})$ with $\Sigma_{y|x}=\Sigma_y-\Sigma_{yx}\Sigma_x^{-1}\Sigma_{xy}$ fixed. $Y-LX-X$ forces the law of $Y$ given $X$ to be a function of $LX$, i.e.\ $B_{yx}X$ is a function of $LX$ a.s.; since $X|LX$ has full support on an affine fiber, $B_{yx}$ vanishes on $\ker L$, i.e.\ $B_{yx}=\Omega L$. Conversely $B_{yx}=\Omega L$ makes the conditional law depend on $X$ only through $LX$. The closed form is $\I(X;Y)=-\tfrac12\sum_i\log(1-\rho_i^2)$ \cite{cover06} applied to $(X,Y)$ and $(LX,Y)$; $\rho_i(L)\le\rho_i$ because canonical correlations of a linear image cannot increase.
\end{proof}

%% ---------------------------------------------------------------------------
\section{Separation theory: the sufficiency deficit}\label{sec:problem}
\begin{definition}[Separated-representation scheme]\label{def:scheme}
A scheme is a pair $(F,\cC)$: a measurable budget-free representation $F:\cX\to\Ub$, $U=F(X)$, and a class $\cC_\lambda$ of channels $P_{T|U}$ (compressors). For each $\lambda$ the scheme selects $C_\lambda\in\cC$ and transmits $T\sim C_\lambda(\cdot|U)$; $\cC$ may be one budget-conditioned network or one network per budget.
Since $T$ depends on $X$ only through $U$, $Y-X-U-T$. Define
\begin{equation}
\JF(\lambda)=\inf_{P_{T|U}}\Jl(P_{T|U}\circ F),\qquad
J^{F,\cC}(\lambda)=\inf_{C\in\cC}\Jl(C\circ F).
\end{equation}
where
\(
P_{T|U}\circ F
=P_{T|U}(t|F(x)).
\)
\end{definition}
\begin{definition}[Sufficiency deficit]\label{def:eps}
$\epsF=\I(X;Y|F(X))=\I(X;Y)-\I(F(X);Y)=H(Y|F(X))-H(Y|X)\ \ge0$. $F$ is \emph{sufficient} if $\epsF=0$, i.e.\ $Y-F(X)-X$.
\end{definition}
\begin{definition}[Separation gap]\label{def:gap}
$\Gap_{F,\cC}(\lambda)=J^{F,\cC}(\lambda)-\Js(\lambda)$ and $\Gap_F(\lambda)=\JF(\lambda)-\Js(\lambda)$. $F$ is $(\delta,\cC)$-separable on a budget set $\mathcal B$ if $\sup_{\lambda\in\mathcal B}\Gap_{F,\cC}(\lambda)\le\delta$.
\end{definition}

\begin{theorem}[Gaussian separation is exact]\label{thm:gauss}
Let $(X,Y)$ be jointly Gaussian, $\Sg=LX$ with $L$ invertible, $\Sigma_S=\mathrm{Cov}(\Sg)$, $\tilde S=\Sigma_S^{-1/2}\Sg$, and $Q_S$ the orthogonal matrix of canonical directions of $(\tilde S,Y)$. For $\beta>1$ the compressor $T_\beta=D(\beta)Q_S^{\!\top}\tilde S+\xi$, $\xi\sim\mathcal N(0,\Id)$, with $D(\beta)=\mathrm{diag}(\alpha_1,\dots,\alpha_r,0,\dots)$ as in \eqref{eq:gib-sol}, attains $\Js(\lambda)$ at $\lambda=1/\beta$. Take $F(X)=\Sg=LX$ and $U=F(X)$, and let $\cC_{\rm diag}$ be the class of channels
\[
T=D\,Q_S^{\!\top}\Sigma_S^{-1/2}U+\xi,\qquad D\in\R^{d_x\times d_x}\text{ diagonal (free for each }\lambda),\ \ \xi\sim\mathcal N(0,\Id)\text{ independent of }U.
\]
Hence $\Delta_{F,\cC_{\rm diag}}(\lambda)=0$, for every $\lambda>0$ (for $\lambda<1$ by the construction above, for $\lambda\ge1$ by Proposition~\ref{prop:range}), with $r$ budget-dependent scalars; $D(\beta)$ is an invariant of $(X,Y)$, independent of $L$, and mode $i$ switches on at $\beta_i^c=1/\rho_i^2$.
\end{theorem}
\begin{proof}
$\Sg$ is a bijection of $X$, so channels from $\Sg$ and from $X$ coincide and $\Js$ is the IB value of $(\Sg,Y)$. This pair is jointly Gaussian with the same canonical correlations (invariant under invertible linear maps), so Proposition~\ref{prop:gib} applies in the canonical coordinates $Q_S^\top\tilde S$. For one mode with $\kappa=1-\rho^2$, $a=\alpha^2$: $\I(X;T)=\tfrac12\log(1+a)$, $\I(T;Y)=\tfrac12\log(1+a)-\tfrac12\log(1+a\kappa)$, and stationarity of $\I(X;T)-\beta\I(T;Y)$ in $a$ reads $(1-\beta)(1+a\kappa)+\beta\kappa(1+a)=0$, i.e.\ $a=(\beta\rho^2-1)/(1-\rho^2)$ if $\beta\rho^2>1$; otherwise $a=0$ (the derivative at $0$ is $\tfrac12(1-\beta\rho^2)\ge0$).
\end{proof}

Adding and subtracting $\JF$ gives the exact decomposition
\begin{equation}
\Gap_{F,\cC}(\lambda)=\Gap_F(\lambda)+J^{F,\cC}(\lambda)-\JF(\lambda).
\label{eq:decomp}
\end{equation}
where \(\Gap_F(\lambda)\) is the representation sufficiency and \(J^{F,\cC}(\lambda)-\JF(\lambda)\) is the compressor expressiveness, also both \(\ge 0\). 
The first term is an information-theoretic property of $F$ and the pair $(X,Y)$; the second measures whether a practical class $\cC$ can realize the optimum on $U$. Only the first is characterized by $\epsF$.

\begin{lemma}[Reduction]\label{lem:red}
Let $U=F(X)$ and let $T$ be generated from $U$ through
a channel $P_{T|U}$. Since $X-U-T$ is a Markov chain and
$U$ is a deterministic function of $X$,
\(
I(X;T)=I(U;T).
\)
Consequently,
\(
J^F(\lambda)
=\inf_{P_{T|U}}
\big[H(Y|T)+\lambda I(U;T)\big]
\)
is the IB value of the pair $(U,Y)$. Since this optimization
restricts the admissible channels relative to the unrestricted
IB problem, $J^F(\lambda)\ge J^\star(\lambda)$.
Moreover, $Y-U-T$ implies $H(Y|T)\ge H(Y|U)$.
Thus, for every $\lambda\ge0$,
\[
J^F(\lambda)\ge H(Y|U)
=H(Y|X)+\epsilon_F.
\]
\end{lemma}

\begin{proof}
$X-U-T$ is Markov and $U$ is a function of $X$, so $\I(X;T)=\I(U;T)+\I(X;T|U)=\I(U;T)$. Compressors on $U$ are channels from $X$, so $\JF\ge\Js$. Finally $\I(T;Y)\le\I(U;Y)$ gives $H(Y|T)\ge H(Y|U)$.
\end{proof}

For a channel $P_{T|X}$, define its $U$-averaged
channel by
\(
P_{T'|U}(\cdot|u)
\triangleq
\mathbb E[P_{T|X}(\cdot|X)\mid U=u],
\, U=F(X).
\)
Generate $T'$ from $U$ using this channel, so that
$Y-U-T'$ forms a Markov chain.

Define
\[
\varphi(\tau)\triangleq
\begin{cases}
\tau\log(K-1)+h_{\rm b}(\tau),
& \tau\le 1-1/K,\\
\log K,&\text{otherwise},
\end{cases}
\]
and
\(
\Phi(\epsilon)\triangleq
\varphi\big(\sqrt{\epsilon/2}\big).
\)
Moreover, $\varphi$ is concave and nondecreasing, $\Phi(0)=0$, and $|H(P)-H(Q)|\le\varphi(\TV(P,Q))$ for distributions on $\cY$ \cite{zhang07,audenaert07}, where $\TV$ is total variation distance and $\hb(t)=-t\log t-(1-t)\log(1-t)$ is binary entropy.

\begin{lemma}[Coupling]\label{lem:coupling}
Let $Y-X-T$ and $T'$ the $U$-averaged output. Then (i) $(U,T')\overset{d}{=}(U,T)$; (ii) $\I(U;T')=\I(U;T)\le\I(X;T)$; (iii) $\I(T;Y|U)\le\epsF$ and $\I(T;Y)-\I(T';Y)\le\Phi(\epsF)$. If $\epsF=0$ then $(T',Y)\overset{d}{=}(T,Y)$.
\end{lemma}
\begin{proof}
(i) is the definition; (ii) follows from (i) and data processing. (iii) $Y-X-T$ and $U=F(X)$ give $\I(Y;X,T|U)=\I(Y;X|U)+\I(Y;T|X,U)=\epsF$, and $\I(Y;X,T|U)\ge\I(Y;T|U)$. Let $P=P_{UTY}$ and $Q=P_{UT}P_{Y|U}$ (under $Q$, $(T,Y)$ has the law of $(T',Y)$). Then $P_T=Q_T$, $P_Y=Q_Y$ and $\KL(P_{UTY}\|Q_{UTY})=\I_P(T;Y|U)\le\epsF$. By data processing $\E_t\KL(P_{Y|t}\|Q_{Y|t})=\KL(P_{TY}\|Q_{TY})\le\epsF$, so by Pinsker and Jensen $\E_t\TV(P_{Y|t},Q_{Y|t})\le\sqrt{\epsF/2}$. With equal $T$-marginals, $\I_P(T;Y)-\I_Q(T;Y)=\E_t[H(Q_{Y|t})-H(P_{Y|t})]\le\E_t\varphi(\TV_t)\le\varphi(\E_t\TV_t)\le\Phi(\epsF)$. If $\epsF=0$ then $P=Q$.
\end{proof}

\subsection{Exactness, stability and necessity}
\begin{theorem}[Exact separation under sufficiency]\label{thm:exact}
If $\epsF=0$ then $\JF(\lambda)=\Js(\lambda)$ for every $\lambda\ge0$ (arbitrary $\cY$, with $\Jl=-\I(T;Y)+\lambda\I(X;T)$): the $\beta$-free representation loses nothing and all $\beta$-dependence can be placed in a compressor acting on $U$.
\end{theorem}
\begin{proof}
For any $P_{T|X}$ form $T'$ as in Lemma~\ref{lem:coupling}: $(T',Y)\overset d=(T,Y)$ and $\I(U;T')\le\I(X;T)$, so $\Jl(T'\circ F)\le\Jl(P_{T|X})$. Take the infimum.
\end{proof}

\begin{proposition}[Separation in the IB iteration]
Let $\cX,\cY$ and the alphabet of $T$ be finite, $p_x=P_{Y|X=x}$ with $p(x)>0$, and consider the
self-consistent IB iteration
\(
q(t|x)\propto q(t)\exp\{-\KL(p_x\|q_{Y|t})/\lambda\},\,
\)
with
\(
q(t)=\sum_xp(x)q(t|x),\,
\)
and
\(
q(y|t)=\sum_xp_x(y)\,q(x|t),
\)
with $q(x|t)=p(x)q(t|x)/q(t)$, started from any $q(t|x)$ with $q(t)>0$ and
$q(y|t)>0$. Let $\pi_u=P_{Y|U=u}$ for $U=F(X)$. For every $x,t$,
\(
q(t|x)\propto q(t)\exp\{-[\KL(\pi_{F(x)}\|q_{Y|t})+c(x,t)]/\lambda\},\,
c(x,t)=\sum_y\big(p_x-\pi_{F(x)}\big)(y)\log\frac{\pi_{F(x)}(y)}{q(y|t)},
\)
with $\E[c(X,t)\mid U]=0$. If $\epsF=0$, then $c\equiv0$ and, after the first
update, $q(t|x)$ depends on $x$ only through $F(x)$; from then on the iteration
coincides with the same iteration run on $(U,Y)$.
\end{proposition}
\begin{proof}
$\KL(p_x\|q_{Y|t})=\KL(p_x\|\pi_u)+\KL(\pi_u\|q_{Y|t})+c(x,t)$ by expanding
$\log(p_x/q_t)=\log(p_x/\pi_u)+\log(\pi_u/q_t)$; the first term does not depend
on $t$ and cancels in the normalization. Since $\E[p_X|U=u]=\pi_u$, we get
$\E[c(X,t)|U=u]=0$. If $\epsF=0$ then $p_x=\pi_{F(x)}$ a.s., so $c\equiv0$ and
the energy depends on $x$ only through $u$. Then $\sum_{x\in F^{-1}(u)}q(x|t)=q(u|t)$,
hence $q(t)=\sum_up(u)q(t|u)$ and $q(y|t)=\sum_u\pi_u(y)q(u|t)$, which are the
updates on $(U,Y)$.
\end{proof}

\begin{theorem}[Uniform stability]\label{thm:stab}
For every $\lambda\ge0$, $0\le\Gap_F(\lambda)\le\Phi(\epsF)$. Equivalently, for every rate $\mathsf r$, $R^\star(\mathsf r)-\Phi(\epsF)\le R^F(\mathsf r)\le R^\star(\mathsf r)$, where $R^F$ is the relevance--rate function restricted to compressors on $U$.
\end{theorem}
\begin{proof}
By Lemma~\ref{lem:coupling}, $\Jl(T'\circ F)=H(Y|T')+\lambda\I(U;T')\le H(Y|T)+\Phi(\epsF)+\lambda\I(X;T)$; take the infimum over $P_{T|X}$. The frontier statement uses the same construction at fixed rate.
\end{proof}

\begin{proposition}[Necessity of sufficiency]\label{prop:lower}
For any channel $P_{T_0|X}$ and any $\lambda$, $\Gap_F(\lambda)\ge\epsF-[\I(X;Y)-\I(T_0;Y)]-\lambda\I(X;T_0)$. If $X$ is finite, $\epsF-\lambda H(X)\le\Gap_F(\lambda)\le\epsF+\lambda H(X)$, so $\lim_{\lambda\downarrow0}\Gap_F(\lambda)=\epsF$ and
\begin{equation}
\epsF\ \le\ \sup_{\lambda}\Gap_F(\lambda)\ \le\ \Phi(\epsF).
\label{eq:sandwich}
\end{equation}
Hence for finite $X$, $\sup_\lambda\Gap_F=0$ iff $\epsF=0$. For general $\cX$ the same holds with $\sup_\lambda$ replaced by $\limsup_{\lambda\downarrow0}$.
\end{proposition}
\begin{proof}
$\JF\ge H(Y|X)+\epsF$ by Lemma~\ref{lem:red}, and $\Js\le H(Y|T_0)+\lambda\I(X;T_0)=H(Y|X)+[\I(X;Y)-\I(T_0;Y)]+\lambda\I(X;T_0)$. For finite $X$ take $T_0=X$: the bracket vanishes and $\I(X;T_0)=H(X)$. For the upper bound, $T=U$ is admissible, so $\JF\le H(Y|U)+\lambda H(U)$, while $\Js\ge H(Y|X)$; hence $\Gap_F\le\epsF+\lambda H(U)\le\epsF+\lambda H(X)$. The right end of \eqref{eq:sandwich} is Theorem~\ref{thm:stab}. For general $\cX$ let $T_0$ be a finite quantization of $X$ with $\I(T_0;Y)\to\I(X;Y)$ (Gelfand--Yaglom--Perez; $\I(X;Y)\le H(Y)<\infty$); each $\I(X;T_0)=H(T_0)<\infty$, so $\liminf_{\lambda\downarrow0}\Gap_F\ge\epsF-[\I(X;Y)-\I(T_0;Y)]$ for every quantization.
\end{proof}
\begin{proposition}[Budget range]\label{prop:range}
For $\lambda\ge1$ and finite $\cY$, $\Js(\lambda)=\JF(\lambda)=H(Y)$ (with the continuous-$Y$ convention $\Jl=-\I(T;Y)+\lambda\I(X;T)$ both equal $0$), and hence $\Gap_F(\lambda)=0$ for every representation $F$.
\end{proposition}
\begin{proof}
By data processing $\I(X;T)\ge\I(T;Y)$, so $\Jl\ge H(Y)-\I(T;Y)+\lambda\I(T;Y)\ge H(Y)$ for $\lambda\ge1$; the constant $T$ is a compressor on $U$ and attains $H(Y)$.
\end{proof}
The deficit is thus felt only at budgets $\lambda<1$, and the worst case occurs in the high-rate (small-$\lambda$) regime. Also, as $\epsilon\to0$, $\Phi(\epsilon)=\Theta\big(\sqrt{\epsilon}\,\log(1/\epsilon)\big)$, because $\hb(t)=t\log(1/t)+O(t)$ as $t\downarrow0$ with $t=\sqrt{\epsilon/2}$, so the binary-entropy term of $\varphi$ dominates; the square-root dependence enters through the Pinsker-based coupling argument. 

\subsection{Certificates for adaptation and pooled domains}\label{sec:adapt}
\begin{corollary}[Adaptation certificate]\label{cor:transfer}
Let $F$ be learned on $P_{XY}$ and deployed on a target $P'_{XY}$ (same alphabets) with $\epsF(P')=\I_{P'}(X;Y|F(X))$. For every $\lambda$, retraining \emph{only} a compressor on $U$ under $P'$ attains $\JF_{P'}(\lambda)\le\Js_{P'}(\lambda)+\Phi(\epsF(P'))$, and exactly $\Js_{P'}(\lambda)$ if $\epsF(P')=0$.
\end{corollary}
\begin{proof}
Apply Theorems~\ref{thm:exact} and~\ref{thm:stab} to $P'$.
\end{proof}

\begin{proposition}[Transfer of the deficit under covariate shift]\label{prop:ud:shift}
Let $P'_{Y|X}=P_{Y|X}$, $P'_X\ll P_X$ and $\varrho:=\|dP'_X/dP_X\|_\infty<\infty$. For the same $F$, $\epsilon_F(P')\le\varrho\,\epsilon_F(P)$.
\end{proposition}
\begin{proof}
$\I(X;Y|F(X))=\min_Q\E\KL(P_{Y|X}\|Q_{Y|F(X)})$. Choosing the source posterior $P_{Y|F(X)}$ gives $\epsilon_F(P')\le\E_{P'}\KL(P_{Y|X}\|P_{Y|F(X)})=\int\frac{dP'_X}{dP_X}\KL(\cdot)\,dP_X\le\varrho\,\epsilon_F(P)$.
\end{proof}

\begin{proposition}[Pooled multi-domain deficit]\label{prop:pooled}
Let $S$ index domains with weights $w_s$ and $P=\sum_sw_sP_s$ on a common $(\cX,\cY)$, under the assumption that the domain is a deterministic function of the input ($S=\mathsf{dom}(X)$ a.s.). For any representation $F$,
\(
\epsF(P)=\sum_sw_s\,\epsF(P_s)+\I_P(Y;S\mid F(X)),\) hence \( \Gap_F(\lambda;P)\le\Phi\Big(\sum_sw_s\epsF(P_s)+\I_P(Y;S|F(X))\Big).
\)
\end{proposition}
\begin{proof}
$H(Y|X)=H(Y|X,S)=\sum_sw_sH_s(Y|X)$ because $S$ is a function of $X$. Also $H(Y|F)=H(Y|F,S)+\I(Y;S|F)$ and $H(Y|F,S)=\sum_sw_sH_s(Y|F)$. Subtract, use Definition~\ref{def:eps} per domain, and apply Theorem~\ref{thm:stab} to $P$.
\end{proof}
\begin{remark}\label{rem:pooled}
$\I(Y;S|F(X))$ vanishes iff $P_s(Y|F(X))$ does not depend on $s$, \( P_s(Y\mid F(X)=u)=P(Y\mid F(X)=u)\); if the meaning of a label index differs across domains, a representation must implicitly encode the domain to make it vanish. With a \emph{domain-specific} compressor, Theorem~\ref{thm:stab} per domain and concavity of $\Phi$ give $\sum_sw_s[\JF_{P_s}-\Js_{P_s}]\le\sum_sw_s\Phi(\epsilon_s)\le\Phi(\sum_sw_s\epsilon_s)$. With compressor classes, \eqref{eq:decomp} per domain reads $J^{F,\cC}_{P_s}-\Js_{P_s}\le\Phi(\epsilon_s)+\delta_{\mathrm{comp},s}(\lambda)$ with $\delta_{\mathrm{comp},s}=J^{F,\cC}_{P_s}-\JF_{P_s}$, where $\epsilon_s:=\epsF(P_s)$, and, since $P_s(Y|X)=P(Y|X)$ under $S=\mathsf{dom}(X)$ and $w_s>0$, covariate shift gives $\epsilon_s\le\varrho_s\,\epsilon_F(P)$ with $\varrho_s:=\|dP_{s,X}/dP_X\|_\infty\le 1/w_s$ (Proposition~\ref{prop:ud:shift}).
\end{remark}

%% ---------------------------------------------------------------------------
\subsection{Change of representation: an exact identity for the deficit}\label{sec:sens}
Let $\cY$ be finite and $F,\Ft$ measurable maps from $\cX$ into a common space $\Ub$ (below, the interface space). Put $\pi_F(u)=P(Y|F(X)=u)$. A \emph{predictor} is a kernel $\psi$ from $\Ub$ to $\Pset$. Define
\begin{align}
\mathrm{CE}(\psi\circ F)&\triangleq \E\big[-\log\psi\big(Y\mid F(X)\big)\big],\nonumber\\
\eta_\psi(F)& \triangleq \E\,\KL\big(\pi_F(F(X))\,\|\,\psi(F(X))\big)=\mathrm{CE}(\psi\circ F)-H(Y|F(X))\ \ge0,\nonumber\\
\chi_\psi(F\!\to\!\Ft)& \triangleq \mathrm{CE}(\psi\circ\Ft)-\mathrm{CE}(\psi\circ F).
\label{eq:sens-defs}
\end{align}
$\eta_\psi$ is the excess risk of $\psi$ on its representation (zero iff $\psi=\pi_F$ a.s.); $\chi_\psi$ is the only term of indefinite sign. The second form of $\eta_\psi$ holds because, writing $U_F=F(X)$, 
\(
\E\big[-\log\psi(Y\mid U_F)\,\big|\,U_F\big]=H\big(\pi_F(U_F)\big)+\KL\big(\pi_F(U_F)\,\|\,\psi(U_F)\big).
\)

\begin{theorem}[Exact identity]\label{thm:sens}
For every predictor $\psi$ with finite cross-entropies,
\begin{equation}
\epsilon_{\Ft}-\epsilon_F=\eta_\psi(F)+\chi_\psi(F\!\to\!\Ft)-\eta_\psi(\Ft).
\label{eq:sens}
\end{equation}
where \(\epsilon_{\Ft}=I(X;Y\mid \widetilde F(X))\), hence $\epsilon_{\Ft}\le\epsilon_F+\eta_\psi(F)+\chi_\psi(F\!\to\!\Ft)$, with equality iff $\psi(\Ft(X))=\pi_{\Ft}(\Ft(X))$ a.s.
\end{theorem}
\begin{proof}
For any representation $R\in\{F,\Ft\}$, $\mathrm{CE}(\psi\circ R)=H(Y|R(X))+\eta_\psi(R)$ and $H(Y|R(X))=H(Y|X)+\epsilon_R$ (Definition~\ref{def:eps}). Subtract the two instances.
\end{proof}

\begin{corollary}[What can be certified]\label{cor:sens}
Let $\psi\ge\psi_{\min}>0$ and $b_\psi=\log(1/\psi_{\min})$.
\begin{enumerate}[label=(\alph*),leftmargin=*,itemsep=1pt]
\item $\Gap_{\Ft}(\lambda)\le\Phi\big(\epsilon_F+\eta_\psi(F)+\chi_\psi(F\!\to\!\Ft)\big)$ for all $\lambda\ge0$ (Theorems~\ref{thm:stab}, \ref{thm:sens}; $\Phi$ is nondecreasing and the argument is $\ge\epsilon_{\Ft}\ge0$).
\item $\chi_\psi=\E\big[\log\psi(Y\mid F(X))-\log\psi(Y\mid\Ft(X))\big]$ with summands in $[-b_\psi,b_\psi]$. For $n$ i.i.d.\ labeled samples $(x_i,y_i)_{i=1}^n$ independent of $(F,\Ft,\psi)$, let
$\hat\chi_\psi=\frac1n\sum_{i=1}^n\big[\log\psi(y_i\mid F(x_i))-\log\psi(y_i\mid\Ft(x_i))\big]$. With probability at least $1-\delta$,
\[
\chi_\psi\le\hat\chi_\psi+b_\psi\sqrt{2\log(1/\delta)/n}\qquad(\text{Hoeffding}).
\]
\end{enumerate}
\end{corollary}
\begin{proof}
(a) $\Gap_{\Ft}\le\Phi(\epsilon_{\Ft})$ by Theorem~\ref{thm:stab} applied to $\Ft$, then Theorem~\ref{thm:sens}. (b) One-sided Hoeffding with range $2b_\psi$: $\Pr(\hat\chi_\psi-\chi_\psi\le-t)\le e^{-nt^2/(2b_\psi^2)}$; set the right side to $\delta$.
\end{proof}

\emph{The interface used.} We use a Convolutional Neural Network (CNN) feature $V=F_\theta(X)\in\R^{32}$ and the compressor receives $U:=Z=\mathrm{softmax}(\ell_\vartheta(V))\in\R^{10}$, not $V$. The representation to which the theory applies is $F_{\rm int}=\mathrm{softmax}\circ\ell_\vartheta\circ F_\theta$ (with the checkpoint used in the experiments), and
\begin{equation}
\epsilon_U=\epsilon_V+\I(V;Y\mid U),\qquad \epsilon_V=\I(X;Y\mid V),\quad \epsilon_U=\I(X;Y\mid U),
\end{equation}
because $U$ is a function of $V$. The smoothed predictor $\psi_\gamma(z)=(1-\gamma)z+\gamma\,\mathbf 1/K$ satisfies $\psi_\gamma\ge\gamma/K$ (so $b_\psi=\log(K/\gamma)$) and is introduced only to state the certificate; it is not evaluated in the experiments.

\section{Decoder error and decoder sharing across domains}\label{sec:dec}
Fix a scheme $(F,\cC)$ and a domain $s$ with law $P_s$. Since $T\sim C(\cdot|U)$ with $U=F(X)$, the chain $Y-U-T$ holds under $P_s$. For a decoder, i.e.\ a kernel $t\mapsto g(\cdot|t)\in\Pset$, write $\mathrm{CE}_s(g)\triangleq \E_{P_s}[-\log g(Y|T)]$ and $P_s(y|t)=P_s(Y{=}y|T{=}t)$. 

\begin{theorem}[Decoder-error decomposition]\label{thm:dec-id}
For every domain $s$ and every decoder $g$ (both sides may be
$+\infty$),
\begin{equation}
\mathrm{CE}_s(g)-H_s(Y|X)
=
\epsF(P_s)
+\I_s(U;Y|T)
+\E_{P_s(T)}
\KL\big(P_s(Y|T)\,\|\,g(\cdot|T)\big),
\label{eq:dec-id}
\end{equation}
where \(\epsF(P_s)\) quantifies representation loss,
\(\I_s(U;Y|T)\) quantifies compression loss, and the term of 
\(\E_{P_s(T)}\KL\big(P_s(Y|T)\,\|\,g(\cdot|T)\big)\)
quantifies decoder mismatch. All three terms are nonnegative.
\end{theorem}

\begin{proof}
$\mathrm{CE}_s(g)=H_s(Y|T)+\E\,\KL(P_s(Y|T)\|g(\cdot|T))$. Next, $H_s(Y|T)-H_s(Y|X)=\I(X;Y)-\I(T;Y)=[\I(X;Y)-\I(U;Y)]+[\I(U;Y)-\I(T;Y)]$. Since $U$ is a function of $X$, the first bracket is $\I(X;Y|F(X))=\epsF(P_s)$. Since $Y-U-T$, $\I(U,T;Y)=\I(U;Y)$, and expanding $\I(U,T;Y)$ the other way gives $\I(U;Y)-\I(T;Y)=\I(U;Y|T)$.
\end{proof}

\begin{theorem}[Pooled decoder]\label{thm:dec-crit}
Let $S\sim w$ index domains, with $(X,Y)\sim P_s$ and $T\sim C_s(\cdot|F(X))$ given $S=s$, and let $P$ be the induced law of $(S,T,Y)$. Let $g_{\rm pool}(\cdot|t)=P(Y{=}\cdot|T{=}t)$ be the pooled Bayes decoder and $g_s^\star(\cdot|t)=P_s(\cdot|t)$ the domain-specific Bayes decoder, so $\mathrm{CE}_s(g_s^\star)=H_s(Y|T)$. Then
\begin{enumerate}[label=(\roman*),leftmargin=*,itemsep=1pt]
\item $\sum_sw_s\big[\mathrm{CE}_s(g_{\rm pool})-\mathrm{CE}_s(g_s^\star)\big]=\sum_sw_s\E_{P_s(T)}\KL\big(P_s(Y|T)\|g_{\rm pool}(\cdot|T)\big)=\I(Y;S\mid T)$;
\item for every shared decoder $g$, $\sum_sw_s\mathrm{CE}_s(g)=\sum_sw_s\mathrm{CE}_s(g_s^\star)+\I(Y;S|T)+\E_{P(T)}\KL(P(Y|T)\|g(\cdot|T))$, so $g_{\rm pool}$ is the best shared decoder and no shared decoder has average excess below $\I(Y;S|T)$;
\item the following are equivalent: $\I(Y;S|T)=0$; $P_s(\cdot|t)=P(\cdot|t)$ for $P_s(T)$-a.e.\ $t$ and every $s$ with $w_s>0$; $g_{\rm pool}$ is Bayes-optimal on every domain; some single decoder is Bayes-optimal on every domain.
\end{enumerate}
\end{theorem}
\begin{proof}
The first equality of (i) is the decomposition $\mathrm{CE}_s(g)=H_s(Y|T)+\E_{P_s(T)}\KL(P_s(Y|T)\|g(\cdot|T))$ applied with $g=g_{\rm pool}$. For the second, $\I(Y;S|T=t)=\sum_sP(s|t)\KL(P_s(Y|t)\|P(Y|t))$, and averaging over $t\sim P(T)$ gives $\sum_sw_s\E_{P_s(T)}\KL(P_s(Y|T)\|g_{\rm pool}(\cdot|T))=\I(Y;S|T)$. (ii) $\sum_sw_s\mathrm{CE}_s(g)=\E[-\log g(Y|T)]=H(Y|T)+\E_{P(T)}\KL(P(Y|T)\|g)$, and $H(Y|T)=H(Y|T,S)+\I(Y;S|T)$ with $H(Y|T,S)=\sum_sw_sH_s(Y|T)=\sum_sw_s\mathrm{CE}_s(g_s^\star)$. (iii) The first three statements are equivalent by (i) and $\KL\ge0$ with equality iff the arguments agree. If a single $g$ is Bayes-optimal on every $s$, then $g(\cdot|t)=P_s(\cdot|t)$ for $P_s(T)$-a.e.\ $t$, for every $s$; hence for $P$-a.e.\ $t$ all $s$ with $P(s|t)>0$ share the same conditional, which then equals $P(\cdot|t)$.
\end{proof}

\begin{proposition}[Relation to the alignment term of Proposition~\ref{prop:pooled}]\label{prop:dec-bound}
Assume $F$ and the compressor do not depend on $s$, so $T\perp(Y,S)\mid U$. Then
\(
\I(Y;S\mid T)\ \le\ \I(Y;S\mid F(X))+\big[\I(U;S)-\I(T;S)\big],\, \I(U;S)-\I(T;S)\ge0.
\)
\end{proposition}
\begin{proof}
$\I(Y;S|T)\le\I(Y,U;S|T)=\I(U;S|T)+\I(Y;S|U,T)$. Since $T\perp(Y,S)|U$, $\I(Y;S|U,T)=\I(Y;S|U)$ and $\I(T;S|U)=0$, so $\I(U;S|T)=\I(U,T;S)-\I(T;S)=\I(U;S)-\I(T;S)\ge0$ (data processing).
\end{proof}

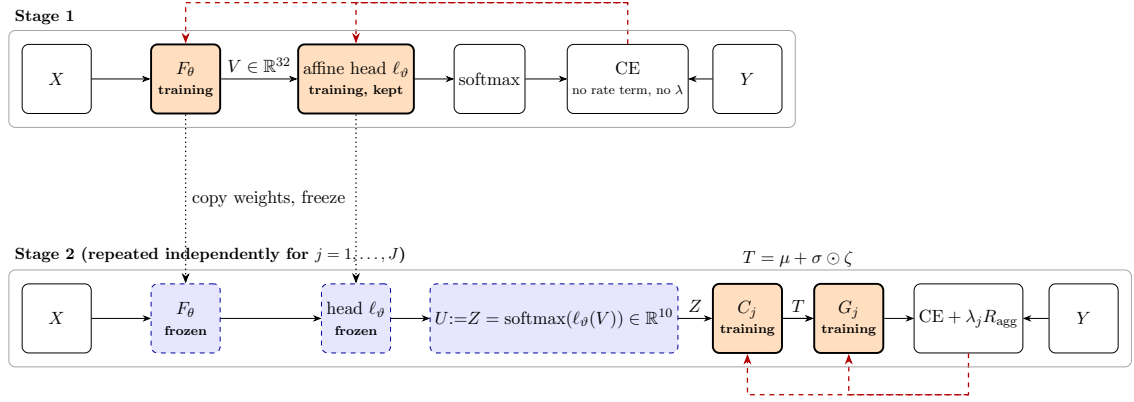
\begin{figure*}[t]
\centering
\resizebox{0.9\textwidth}{!}{%
\begin{tikzpicture}[>=Stealth,
  box/.style={draw,rounded corners,align=center,minimum height=15mm,minimum width=15mm,inner sep=3pt},
  train/.style={box,fill=orange!25,very thick},
  frozen/.style={box,fill=blue!10,dashed,draw=blue!60!black},
  grad/.style={->,red!70!black,dashed,thick},
  panel/.style={draw=black!35,rounded corners,inner sep=8pt}]
% ================= STAGE 1: F + linear head, cross-entropy only
\node[box]   (x1)  at (0,0)    {$X$};
\node[train] (F1)  at (2.8,0)  {$F_\theta$\\[-1pt]{\scriptsize\bfseries training}};
\node[train] (hd1) at (6.5,0)  {affine head $\ell_\vartheta$\\[-1pt]{\scriptsize\bfseries training, kept}};
\node[box]   (sm1) at (9.4,0)  {softmax};
\node[box]   (ce1) at (12.4,0) {$\mathrm{CE}$\\[-1pt]{\scriptsize no rate term, no $\lambda$}};
\node[box]   (y1)  at (15.0,0) {$Y$};
\draw[->] (x1)--(F1);
\draw[->] (F1)--node[above]{$V\in\R^{32}$}(hd1);
\draw[->] (hd1)--(sm1);
\draw[->] (sm1)--(ce1);
\draw[->] (y1)--(ce1);
\draw[grad] (ce1.north) -- ++(0,0.9) -| (hd1.north);
\draw[grad] ($(ce1.north)+(0,0.9)$) -| (F1.north);
\node[panel,fit=(x1)(F1)(hd1)(sm1)(ce1)(y1),
  label={[anchor=south west,font=\small\bfseries]north west:Stage 1}] (p1) {};
% ================= STAGE 2: one independent (C_j,G_j) per budget
\node[box]    (x2)  at (0,-5.2)    {$X$};
\node[frozen] (F2)  at (2.8,-5.2)  {$F_\theta$\\[-1pt]{\scriptsize\bfseries frozen}};
\node[frozen] (hd2) at (6.5,-5.2)  {head $\ell_\vartheta$\\[-1pt]{\scriptsize\bfseries frozen}};
\node[frozen] (Z)   at (10.8,-5.2)  {$U{:=}Z=\mathrm{softmax}(\ell_\vartheta(V))\in\R^{10}$};
\node[train]  (C)   at (15,-5.2) {$C_j$\\[-1pt]{\scriptsize\bfseries training}};
\node[train]  (G)   at (17.2,-5.2) {$G_j$\\[-1pt]{\scriptsize\bfseries training}};
\node[box]    (ls)  at (19.8,-5.2) {$\mathrm{CE}+\lambda_jR_{\rm agg}$};
\node[box]    (y2)  at (22.3,-5.2) {$Y$};
\draw[->] (x2)--(F2);
\draw[->] (F2)--(hd2);
\draw[->] (hd2)--(Z);
\draw[->] (Z)--node[above]{$Z$}(C);
\draw[->] (C)--node[above] {$T$}(G);
\draw[->] (C)--node[above, yshift=28pt] {$T=\mu+\sigma\odot\zeta$}(G);
\draw[->] (G)--(ls);
\draw[->] (y2)--(ls);
\draw[grad] (ls.south) -- ++(0,-0.9) -| (G.south);
\draw[grad] ($(ls.south)+(0,-0.9)$) -| (C.south);
\node[panel,fit=(x2)(F2)(hd2)(Z)(C)(G)(ls)(y2),
  label={[anchor=south west,font=\small\bfseries]north west:Stage 2 (repeated independently for $j=1,\dots,J$)}] (p2) {};

% ================= link between stages
\draw[->,thick,dotted] (F1.south) -- node[right,align=left]{copy weights, freeze} (F2.north);
\draw[->,thick,dotted] (hd1.south) -- (hd2.north);
\end{tikzpicture}}
\caption{\textbf{Stage~1:} $F_\theta$ and an affine head $\ell_\vartheta$ are trained with cross-entropy only (no rate term, no $\lambda$). Both are then frozen, and the head is \emph{kept}. \textbf{Stage~2:} the ten-dimensional softmax output $Z$ ($=U$) is cached for each training image. For each budget $\lambda_j$, an independent compressor--decoder pair $(C_j,G_j)$ is trained on $\mathrm{CE}+\lambda_jR_{\rm agg}$ with $\lambda_j$ fixed and never fed to the networks. A new budget therefore costs one additional $(C,G)$ fit. Orange boxes denote trainable components, blue dashed boxes denote frozen/cached components, and dashed red arrows indicate gradient flow. }
\label{fig:stages}
\end{figure*}

\begin{algorithm}[t]
\caption{Separated pipeline (E1, E2)}\label{alg:sep}
\begin{algorithmic}[1]
\Require training pool $\cD$ (with validation split); budgets $\lambda_1,\dots,\lambda_J$
\Ensure frozen $F_{\rm int}$ and pairs $(C_j,G_j)_{j=1}^{J}$
\Statex \textbf{Stage 1: representation (no $\lambda$)}
\State Train $F_\theta$ and affine head $\ell_\vartheta$ with cross-entropy only; freeze both
\State $U \leftarrow \mathrm{softmax}(\ell_\vartheta(F_\theta(X)))$ \Comment{cached once, reused for all budgets}
\Statex \textbf{Stage 2: rate control (one pair per budget)}
\For{$j=1,\dots,J$}
  \State Initialize $(C_j,G_j)$
  \Repeat
    \State $T \leftarrow \mu+\sigma\odot\zeta$, $\;(\mu,\sigma)=C_j(U)$, $\;\zeta\sim\mathcal N(0,\Id)$
    \State Update $(C_j,G_j)$ on $\mathrm{CE}(G_j(T),Y)+\lambda_j R_{\rm agg}$
  \Until{best validation loss}
  \State Fit and freeze the aggregate prior $(\bar\mu_j,\sigma_{T,j}^2)$ for evaluation
\EndFor
\State \textbf{Evaluate} each $(C_j,G_j)$ on test data: accuracy, CE, $R_{\rm agg}$, $\mathcal L_{\lambda_j}$, ECE
\end{algorithmic}
\end{algorithm}

\begin{table*}[t]
\centering
\scriptsize
\setlength{\tabcolsep}{4pt}
\caption{Architecture and computational cost over the full grid of
eight prespecified $\lambda$ values. MACs and training times are averaged over 10 seeds from the reported run. FC stands for fully connected (linear) layer.}
\label{tab:architecture_cost}
\resizebox{\textwidth}{!}{%
\begin{tabular}{@{}llp{6.1cm}rr@{}}
\toprule
Experiments & Architecture &
Network structure and rate dependence &
Total training MACs &
Training time (s) \\
\midrule

E1 / E2 &
Separated FCG &
E1 trains one CNN $F$ per dataset; E2 trains one CNN
on the pooled data. Each trained representation is reused
across all eight budgets. For each $\lambda$, the compressor
$C$: $10\!\to\!128\!\to\!128\!\to\!64$ (3 FC layers),
and decoder $G$: $32\!\to\!128\!\to\!10$ (2 FC layers)
are trained separately. Each CNN has 3 Conv layers,
2 FC layers, and a 10-class linear head. &
$3.039\!\times\!10^{13}$ /
$3.039\!\times\!10^{13}$ &
141.3 / 138.2 \\

\midrule
E3 / E4 &
Independent VIB &
Per model: CNN encoder (3 Conv + 2 FC, outputting $\mu$ and variance)
and decoder ($32\!\to\!128\!\to\!10$, 2 FC).
A separate end-to-end model is trained for each $\lambda$. &
$2.347\!\times\!10^{14}$ /
$2.347\!\times\!10^{14}$ &
292.9 / 294.1 \\

\midrule
E5 / E6 &
FiLM-conditioned VIB &
One rate-conditioned network: 3 Conv, 2 encoder FC,
4 FiLM FC, 2 rate-embedding FC, and 2 decoder FC.
The network is conditioned on $\lambda$ throughout. &
$2.355\!\times\!10^{14}$ /
$2.355\!\times\!10^{14}$ &
487.3 / 489.3 \\

\midrule
E7 / E8 &
OIB &
DFT input (3264 real coordinates); pretrained MLP
with 3 FC layers ($3264\!\to\!128\!\to\!128\!\to\!10$);
closed-form Gaussian-IB compressor; one 2-FC decoder per
budget ($n_z\!\to\!128\!\to\!10$).
The DFT itself is excluded from the MAC count. &
$1.830\!\times\!10^{12}$ /
$1.745\!\times\!10^{12}$ &
76.4 / 73.9 \\

\bottomrule
\end{tabular}%
}

\end{table*}

\begin{table}[t]
\centering\small
\caption{Test metrics averaged over the eight prespecified budgets (best-validation-epoch weights); mean $\pm$ 95\% CI half-width over 10 seeds, all rows from the same run. Here $\overline{\mathcal L}:=\frac18\sum_{j=1}^{8}\big(\mathrm{CE}_j+\lambda_jR_{{\rm agg},j}\big)$, for 8 given \(\beta\), is the budget-averaged surrogate Lagrangian over the eight budgets, with $\mathrm{CE}_j,R_{{\rm agg},j}$ the test values at $\lambda_j$. ``macro'': unweighted mean over the four datasets; ``pooled test'': size-weighted pooled test set (pooled experiments only: E2, E4, E5, E8). E5/E6: one FiLM rate-conditioned VIB (step-matched to E3); E7/E8: OIB, bank-of-NN.}
\label{tab:main}
\resizebox{\textwidth}{!}{%
\begin{tabular}{@{}llccccc@{}}\toprule
Experiment & Aggregation & Acc (\%) & CE (nats) & $R_{\rm agg}$ (nats) & $\overline{\mathcal L}$ & ECE\\\midrule
E1 separated, per domain & macro & $65.59\pm0.49$ & $1.1709\pm0.0078$ & $6.18\pm0.03$ & $1.4874\pm0.0064$ & $0.0862\pm0.0034$\\
E3 VIB, per domain & macro & $57.30\pm0.23$ & $1.1806\pm0.0057$ & $6.96\pm0.03$ & $1.4139\pm0.0043$ & $0.0511\pm0.0015$\\
E6 FiLM-conditioned VIB, per domain & macro & $60.86\pm0.73$ & $1.1227\pm0.0157$ & $4.81\pm0.18$ & $1.4126\pm0.0121$ & $0.0697\pm0.0022$\\
E7 OIB, per domain & macro & $42.72\pm0.39$ & $1.7545\pm0.0057$ & $127.78\pm0.49$ & $6.6694\pm0.0391$ & $0.0876\pm0.0045$\\
\midrule
E2 separated, pooled & macro & $66.64\pm0.75$ & $1.1461\pm0.0135$ & $5.37\pm0.08$ & $1.4646\pm0.0175$ & $0.1003\pm0.0054$\\
E4 VIB, pooled & macro & $56.87\pm0.91$ & $1.2472\pm0.0164$ & $7.22\pm0.12$ & $1.4790\pm0.0065$ & $0.0569\pm0.0026$\\
E5 FiLM-conditioned VIB, pooled & macro & $60.67\pm0.87$ & $1.1778\pm0.0180$ & $5.52\pm0.18$ & $1.4840\pm0.0120$ & $0.0939\pm0.0042$\\
E8 OIB, pooled & macro & $38.98\pm0.27$ & $1.7931\pm0.0045$ & $149.07\pm1.77$ & $7.9666\pm0.1027$ & $0.0824\pm0.0033$\\
\midrule
E2 separated, pooled & pooled test & $67.14\pm0.90$ & $1.1384\pm0.0148$ & $5.35\pm0.08$ & $1.4540\pm0.0168$ & $0.0804\pm0.0052$\\
E4 VIB, pooled & pooled test & $57.06\pm0.90$ & $1.2585\pm0.0170$ & $7.26\pm0.14$ & $1.4891\pm0.0073$ & $0.0357\pm0.0017$\\
E5 FiLM-conditioned VIB, pooled & pooled test & $60.51\pm0.93$ & $1.1969\pm0.0202$ & $5.55\pm0.19$ & $1.5007\pm0.0148$ & $0.0817\pm0.0036$\\
E8 OIB, pooled & pooled test & $37.61\pm0.38$ & $1.8314\pm0.0068$ & $138.13\pm1.68$ & $7.3842\pm0.0863$ & $0.0448\pm0.0018$\\
\bottomrule\end{tabular}}
\end{table}

\begin{figure*}[t]
\centering
\begin{tikzpicture}
\begin{groupplot}[group style={group size=3 by 1,horizontal sep=1.6cm},width=0.35\textwidth,height=0.30\textwidth,xmode=log,xmin=0.0035,xmax=0.7,xtick={0.005,0.01,0.02,0.05,0.1,0.2,0.5},xticklabels={0.005,,0.02,0.05,0.1,0.2,0.5},xlabel={$\lambda$},grid=major,grid style={gray!25},tick label style={font=\scriptsize},label style={font=\small},title style={font=\small},legend style={font=\scriptsize,legend columns=8,draw=none}]
\nextgroupplot[ylabel={test accuracy (\%)},ymin=0,ymax=80,title={(a) accuracy},legend to name=legfig1]
\addplot[blue!70!black,thick,mark=o,mark size=1.5pt,error bars/.cd,y dir=both,y explicit,error bar style={line width=0.4pt}] coordinates {(0.005,71.3564) +- (0,0.3169) (0.00965349,71.0652) +- (0,0.3035) (0.018638,70.6709) +- (0,0.3493) (0.0359843,69.7216) +- (0,0.7062) (0.0694748,68.4864) +- (0,0.3482) (0.134135,65.8082) +- (0,0.4577) (0.258974,60.1449) +- (0,0.9058) (0.5,47.4542) +- (0,1.1690)};
\addlegendentry{E1}
\addplot[orange!90!black,thick,mark=square,mark size=1.5pt,error bars/.cd,y dir=both,y explicit,error bar style={line width=0.4pt}] coordinates {(0.005,70.4192) +- (0,0.5126) (0.00965349,70.3329) +- (0,0.5438) (0.018638,70.2126) +- (0,0.5217) (0.0359843,69.8742) +- (0,0.5727) (0.0694748,69.1157) +- (0,0.5712) (0.134135,67.4942) +- (0,0.6790) (0.258974,63.7996) +- (0,0.9726) (0.5,51.9112) +- (0,2.0202)};
\addlegendentry{E2}
\addplot[green!70!black,thick,mark=triangle,dashed,mark size=1.5pt,error bars/.cd,y dir=both,y explicit,error bar style={line width=0.4pt}] coordinates {(0.005,69.7565) +- (0,0.2502) (0.00965349,68.7125) +- (0,0.5021) (0.018638,67.6314) +- (0,0.4010) (0.0359843,65.7857) +- (0,0.4777) (0.0694748,62.0407) +- (0,0.6997) (0.134135,54.2651) +- (0,1.0942) (0.258974,42.5267) +- (0,0.9117) (0.5,27.7101) +- (0,1.0342)};
\addlegendentry{E3}
\addplot[red!70!black,thick,mark=diamond,dashed,mark size=1.5pt,error bars/.cd,y dir=both,y explicit,error bar style={line width=0.4pt}] coordinates {(0.005,68.7027) +- (0,0.3847) (0.00965349,68.3685) +- (0,0.4502) (0.018638,68.0247) +- (0,0.3884) (0.0359843,66.7309) +- (0,0.5252) (0.0694748,65.4432) +- (0,0.3790) (0.134135,61.1179) +- (0,0.8227) (0.258974,43.9369) +- (0,7.0569) (0.5,12.6310) +- (0,0.1220)};
\addlegendentry{E4}
\addplot[brown!80!black,thick,mark=pentagon,dashed, mark size=1.5pt,error bars/.cd,y dir=both,y explicit,error bar style={line width=0.4pt}] coordinates {(0.005,68.0224) +- (0,0.5951) (0.00965349,67.8420) +- (0,0.6108) (0.018638,67.3598) +- (0,0.5941) (0.0359843,66.4006) +- (0,0.6090) (0.0694748,64.5117) +- (0,0.7827) (0.134135,60.8715) +- (0,1.0737) (0.258974,52.9560) +- (0,1.7345) (0.5,37.3979) +- (0,1.9979)};
\addlegendentry{E5}
\addplot[violet!80!black,thick,mark=star,dashed,mark size=1.5pt,error bars/.cd,y dir=both,y explicit,error bar style={line width=0.4pt}] coordinates {(0.005,66.9099) +- (0,0.7516) (0.00965349,66.5957) +- (0,0.7661) (0.018638,66.0354) +- (0,0.7397) (0.0359843,65.0764) +- (0,0.7565) (0.0694748,63.2343) +- (0,0.7475) (0.134135,59.8839) +- (0,0.7818) (0.258974,54.0104) +- (0,0.8900) (0.5,45.1499) +- (0,0.8875)};
\addlegendentry{E6}
\addplot[teal!80!black,thick,mark=x,dashed,mark size=1.5pt,error bars/.cd,y dir=both,y explicit,error bar style={line width=0.4pt}] coordinates {(0.005,43.4382) +- (0,0.2633) (0.00965349,43.4200) +- (0,0.4806) (0.018638,43.7873) +- (0,0.4650) (0.0359843,43.7748) +- (0,0.4586) (0.0694748,43.4401) +- (0,0.4244) (0.134135,42.9711) +- (0,0.5063) (0.258974,41.9164) +- (0,0.5896) (0.5,39.0357) +- (0,0.5891)};
\addlegendentry{E7}
\addplot[gray,thick,mark=+,dashed,mark size=1.5pt,error bars/.cd,y dir=both,y explicit,error bar style={line width=0.4pt}] coordinates {(0.005,40.0599) +- (0,0.5010) (0.00965349,40.1103) +- (0,0.4978) (0.018638,40.4152) +- (0,0.4104) (0.0359843,40.0160) +- (0,0.4416) (0.0694748,39.9553) +- (0,0.4272) (0.134135,39.7101) +- (0,0.3628) (0.258974,38.0825) +- (0,0.1642) (0.5,33.4853) +- (0,0.2145)};
\addlegendentry{E8}
\nextgroupplot[ylabel={test CE (nats)},ymin=0.8,ymax=2.5,title={(b) cross-entropy}]
\addplot[blue!70!black,thick,mark=o,mark size=1.5pt,error bars/.cd,y dir=both,y explicit,error bar style={line width=0.4pt}] coordinates {(0.005,1.0242) +- (0,0.0111) (0.00965349,1.0439) +- (0,0.0090) (0.018638,1.0661) +- (0,0.0082) (0.0359843,1.0992) +- (0,0.0078) (0.0694748,1.1381) +- (0,0.0062) (0.134135,1.1917) +- (0,0.0060) (0.258974,1.2943) +- (0,0.0150) (0.5,1.5099) +- (0,0.0236)};
\addplot[orange!90!black,thick,mark=square,mark size=1.5pt,error bars/.cd,y dir=both,y explicit,error bar style={line width=0.4pt}] coordinates {(0.005,1.0101) +- (0,0.0184) (0.00965349,1.0257) +- (0,0.0197) (0.018638,1.0410) +- (0,0.0168) (0.0359843,1.0656) +- (0,0.0166) (0.0694748,1.1071) +- (0,0.0170) (0.134135,1.1628) +- (0,0.0153) (0.258974,1.2594) +- (0,0.0152) (0.5,1.4968) +- (0,0.0284)};
\addplot[green!50!black,thick,mark=triangle,dashed,mark size=1.5pt,error bars/.cd,y dir=both,y explicit,error bar style={line width=0.4pt}] coordinates {(0.005,0.9014) +- (0,0.0068) (0.00965349,0.9161) +- (0,0.0124) (0.018638,0.9433) +- (0,0.0098) (0.0359843,0.9851) +- (0,0.0119) (0.0694748,1.0695) +- (0,0.0192) (0.134135,1.2324) +- (0,0.0219) (0.258974,1.4904) +- (0,0.0213) (0.5,1.9063) +- (0,0.0327)};
\addplot[red!70!black,thick,mark=diamond,dashed,mark size=1.5pt,error bars/.cd,y dir=both,y explicit,error bar style={line width=0.4pt}] coordinates {(0.005,0.9377) +- (0,0.0136) (0.00965349,0.9509) +- (0,0.0108) (0.018638,0.9728) +- (0,0.0127) (0.0359843,1.0090) +- (0,0.0129) (0.0694748,1.0654) +- (0,0.0148) (0.134135,1.1839) +- (0,0.0220) (0.258974,1.5612) +- (0,0.1565) (0.5,2.2970) +- (0,0.0004)};
\addplot[violet,thick,mark=star,densely dotted,mark size=1.5pt,error bars/.cd,y dir=both,y explicit,error bar style={line width=0.4pt}] coordinates {(0.005,1.0230) +- (0,0.0182) (0.00965349,1.0101) +- (0,0.0186) (0.018638,1.0053) +- (0,0.0183) (0.0359843,1.0136) +- (0,0.0187) (0.0694748,1.0466) +- (0,0.0180) (0.134135,1.1201) +- (0,0.0174) (0.258974,1.2626) +- (0,0.0199) (0.5,1.5005) +- (0,0.0216)};
\addplot[teal,thick,mark=x,dotted,mark size=1.5pt,error bars/.cd,y dir=both,y explicit,error bar style={line width=0.4pt}] coordinates {(0.005,1.9082) +- (0,0.0164) (0.00965349,1.8262) +- (0,0.0144) (0.018638,1.7731) +- (0,0.0122) (0.0359843,1.7354) +- (0,0.0102) (0.0694748,1.6987) +- (0,0.0053) (0.134135,1.6833) +- (0,0.0059) (0.258974,1.6852) +- (0,0.0114) (0.5,1.7255) +- (0,0.0117)};
\addplot[gray,thick,mark=+,dotted,mark size=1.5pt,error bars/.cd,y dir=both,y explicit,error bar style={line width=0.4pt}] coordinates {(0.005,1.8275) +- (0,0.0099) (0.00965349,1.8017) +- (0,0.0048) (0.018638,1.7871) +- (0,0.0057) (0.0359843,1.7690) +- (0,0.0061) (0.0694748,1.7697) +- (0,0.0105) (0.134135,1.7663) +- (0,0.0049) (0.258974,1.7723) +- (0,0.0028) (0.5,1.8509) +- (0,0.0050)};
\addplot[brown!80!black,thick,mark=pentagon,densely dotted,mark size=1.5pt,error bars/.cd,y dir=both,y explicit,error bar style={line width=0.4pt}] coordinates {(0.005,1.0034) +- (0,0.0215) (0.00965349,0.9983) +- (0,0.0163) (0.018638,1.0064) +- (0,0.0141) (0.0359843,1.0282) +- (0,0.0141) (0.0694748,1.0757) +- (0,0.0163) (0.134135,1.1710) +- (0,0.0256) (0.258974,1.3713) +- (0,0.0444) (0.5,1.7683) +- (0,0.0537)};
\nextgroupplot[ylabel={$R_{\rm agg}$ (nats)},ymin=0,ymax=21,title={(c) surrogate rate (OIB omitted)}]
\addplot[blue!70!black,thick,mark=o,mark size=1.5pt,error bars/.cd,y dir=both,y explicit,error bar style={line width=0.4pt}] coordinates {(0.005,13.6373) +- (0,0.1270) (0.00965349,10.4435) +- (0,0.1525) (0.018638,7.9874) +- (0,0.1471) (0.0359843,5.9164) +- (0,0.0674) (0.0694748,4.5375) +- (0,0.0747) (0.134135,3.3286) +- (0,0.0450) (0.258974,2.2965) +- (0,0.0489) (0.5,1.2900) +- (0,0.0478)};
\addplot[orange!90!black,thick,mark=square,mark size=1.5pt,error bars/.cd,y dir=both,y explicit,error bar style={line width=0.4pt}] coordinates {(0.005,11.6959) +- (0,0.3741) (0.00965349,8.4340) +- (0,0.3115) (0.018638,6.5282) +- (0,0.2401) (0.0359843,5.0590) +- (0,0.0684) (0.0694748,4.0902) +- (0,0.0615) (0.134135,3.2321) +- (0,0.0335) (0.258974,2.4396) +- (0,0.0561) (0.5,1.5100) +- (0,0.0834)};
\addplot[green!50!black,thick,mark=triangle,dashed,mark size=1.5pt,error bars/.cd,y dir=both,y explicit,error bar style={line width=0.4pt}] coordinates {(0.005,18.9994) +- (0,0.1977) (0.00965349,12.9007) +- (0,0.1989) (0.018638,8.9067) +- (0,0.1177) (0.0359843,6.1733) +- (0,0.1045) (0.0694748,4.1570) +- (0,0.0702) (0.134135,2.5902) +- (0,0.0878) (0.258974,1.4168) +- (0,0.0481) (0.5,0.5117) +- (0,0.0436)};
\addplot[red!70!black,thick,mark=diamond,dashed,mark size=1.5pt,error bars/.cd,y dir=both,y explicit,error bar style={line width=0.4pt}] coordinates {(0.005,19.2337) +- (0,0.7867) (0.00965349,12.9218) +- (0,0.3894) (0.018638,9.2573) +- (0,0.3037) (0.0359843,6.5548) +- (0,0.1501) (0.0694748,4.6844) +- (0,0.0945) (0.134135,3.3820) +- (0,0.0764) (0.258974,1.7217) +- (0,0.3673) (0.5,0.0002) +- (0,0.0001)};
\addplot[violet,thick,mark=star,densely dotted,mark size=1.5pt,error bars/.cd,y dir=both,y explicit,error bar style={line width=0.4pt}] coordinates {(0.005,8.5557) +- (0,0.4669) (0.00965349,7.5767) +- (0,0.3731) (0.018638,6.4640) +- (0,0.2764) (0.0359843,5.2907) +- (0,0.1977) (0.0694748,4.1188) +- (0,0.1346) (0.134135,3.0498) +- (0,0.0871) (0.258974,2.0887) +- (0,0.0695) (0.5,1.3123) +- (0,0.0515)};
\addplot[brown!80!black,thick,mark=pentagon,densely dotted,mark size=1.5pt,error bars/.cd,y dir=both,y explicit,error bar style={line width=0.4pt}] coordinates {(0.005,9.7198) +- (0,0.6021) (0.00965349,8.8319) +- (0,0.4183) (0.018638,7.5908) +- (0,0.3141) (0.0359843,6.1926) +- (0,0.2228) (0.0694748,4.8767) +- (0,0.1832) (0.134135,3.6157) +- (0,0.1658) (0.258974,2.3135) +- (0,0.1629) (0.5,1.0562) +- (0,0.1058)};
\end{groupplot}
\node[anchor=north] at ($(group c2r1.south)+(0,-1.1cm)$) {\ref{legfig1}};
\end{tikzpicture}
\caption{Four-dataset macro-average over 10 seeds (unweighted over MNIST, SVHN, CIFAR-10, STL-10) test metrics versus the prespecified budget grid ($\lambda$ is given). E1/E2: separated, per-domain/pooled; E3/E4: independent VIBs; E5/E6: one FiLM rate-conditioned VIB; E7/E8: OIB (bank-of-NN, own re-implementation). (a) Accuracy, (b) cross-entropy, (c) surrogate rate $R_{\rm agg}$; OIB has $R_{\rm agg}\approx 128$--$149$ nats (closed-form compressor) and is omitted from (c) for scale.}
\label{fig:results}
\end{figure*}
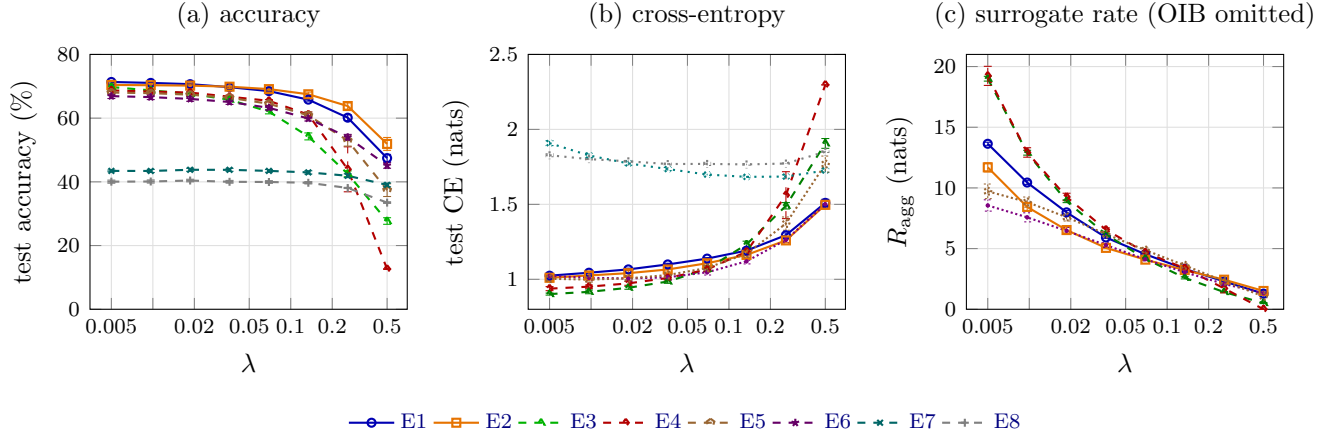

\begin{table}[t]
\centering\scriptsize\setlength{\tabcolsep}{2.6pt}
\caption{Per-dataset results at every $\lambda$ (MNIST, SVHN). Mean over 10 seeds of test accuracy (Acc, \%), test cross-entropy (CE, nats) and surrogate rate ($R=R_{\rm agg}$, nats) at each prespecified $\lambda$.}
\label{tab:app1}
\resizebox{0.8\textwidth}{!}{%
\begin{tabular}{@{}r ccc ccc ccc ccc@{}}\toprule
 & \multicolumn{3}{c}{E1 (sep., per dom.)} & \multicolumn{3}{c}{E2 (sep., pooled)} & \multicolumn{3}{c}{E3 (VIB, per dom.)} & \multicolumn{3}{c}{E4 (VIB, pooled)}\\
\cmidrule(lr){2-4}\cmidrule(lr){5-7}\cmidrule(lr){8-10}\cmidrule(lr){11-13}
$\lambda$ & Acc & CE & $R$ & Acc & CE & $R$ & Acc & CE & $R$ & Acc & CE & $R$\\\midrule
\multicolumn{13}{l}{\textbf{MNIST}}\\
0.0050 & \textbf{96.17} & 0.273 & 11.34 & 96.04 & 0.197 & 11.05 & 95.86 & 0.137 & 14.61 & 95.17 & 0.156 & 18.55\\
0.0097 & 95.93 & 0.295 & 8.85 & \textbf{96.02} & 0.205 & 7.88 & 95.64 & 0.147 & 10.81 & 95.01 & 0.166 & 11.65\\
0.0186 & 95.58 & 0.310 & 7.45 & \textbf{95.98} & 0.217 & 6.09 & 95.06 & 0.164 & 8.12 & 94.25 & 0.197 & 8.50\\
0.0360 & 94.94 & 0.342 & 5.99 & \textbf{95.83} & 0.232 & 4.92 & 94.51 & 0.183 & 6.16 & 93.28 & 0.233 & 6.60\\
0.0695 & 93.50 & 0.370 & 5.04 & \textbf{95.32} & 0.261 & 4.16 & 93.27 & 0.220 & 4.82 & 92.45 & 0.265 & 4.75\\
0.1341 & 91.14 & 0.412 & 4.03 & \textbf{94.14} & 0.307 & 3.49 & 91.12 & 0.284 & 3.79 & 88.30 & 0.398 & 3.71\\
0.2590 & 86.79 & 0.515 & 3.00 & \textbf{90.43} & 0.432 & 2.83 & 86.25 & 0.416 & 2.94 & 63.56 & 1.037 & 2.05\\
0.5000 & \textbf{75.81} & 0.782 & 2.08 & 74.19 & 0.828 & 1.86 & 67.73 & 0.876 & 1.89 & 11.30 & 2.304 & 0.00\\
\midrule
\multicolumn{13}{l}{\textbf{SVHN}}\\
0.0050 & \textbf{73.64} & 1.042 & 14.20 & 70.81 & 1.024 & 11.56 & 69.63 & 0.970 & 20.55 & 67.46 & 0.998 & 19.24\\
0.0097 & \textbf{73.48} & 1.067 & 10.47 & 70.71 & 1.041 & 8.52 & 67.76 & 0.996 & 14.21 & 67.75 & 1.005 & 13.31\\
0.0186 & \textbf{73.08} & 1.094 & 7.92 & 70.47 & 1.061 & 6.40 & 66.01 & 1.043 & 9.65 & 66.61 & 1.036 & 9.52\\
0.0360 & \textbf{72.44} & 1.130 & 6.05 & 69.96 & 1.098 & 5.04 & 63.99 & 1.094 & 6.44 & 65.11 & 1.086 & 6.66\\
0.0695 & \textbf{71.07} & 1.181 & 4.57 & 69.18 & 1.138 & 4.01 & 56.91 & 1.259 & 4.00 & 63.99 & 1.136 & 4.66\\
0.1341 & \textbf{68.30} & 1.247 & 3.36 & 67.30 & 1.203 & 3.16 & 45.89 & 1.529 & 2.16 & 58.42 & 1.303 & 3.28\\
0.2590 & 62.85 & 1.352 & 2.38 & \textbf{62.91} & 1.317 & 2.38 & 27.21 & 2.054 & 0.37 & 41.55 & 1.708 & 1.62\\
0.5000 & 48.80 & 1.605 & 1.25 & \textbf{51.75} & 1.545 & 1.47 & 19.54 & 2.227 & 0.00 & 19.40 & 2.267 & 0.00\\
\bottomrule\end{tabular}}
\end{table}

\begin{table}[t]
\centering\scriptsize\setlength{\tabcolsep}{2.6pt}
\caption{Per-dataset results at every $\lambda$ (CIFAR-10, STL-10, pooled test set). Mean over 10 seeds of test accuracy (Acc, \%), test cross-entropy (CE, nats) and surrogate rate ($R=R_{\rm agg}$, nats) at each prespecified $\lambda$.}
\label{tab:app2}
\resizebox{0.8\textwidth}{!}{%
\begin{tabular}{@{}r ccc ccc ccc ccc@{}}\toprule
 & \multicolumn{3}{c}{E1 (sep., per dom.)} & \multicolumn{3}{c}{E2 (sep., pooled)} & \multicolumn{3}{c}{E3 (VIB, per dom.)} & \multicolumn{3}{c}{E4 (VIB, pooled)}\\
\cmidrule(lr){2-4}\cmidrule(lr){5-7}\cmidrule(lr){8-10}\cmidrule(lr){11-13}
$\lambda$ & Acc & CE & $R$ & Acc & CE & $R$ & Acc & CE & $R$ & Acc & CE & $R$\\\midrule
\multicolumn{13}{l}{\textbf{CIFAR-10}}\\
0.0050 & 58.66 & 1.358 & 15.32 & \textbf{58.74} & 1.401 & 11.95 & 57.98 & 1.222 & 20.76 & 56.07 & 1.287 & 19.23\\
0.0097 & 58.32 & 1.384 & 11.26 & \textbf{58.67} & 1.421 & 8.80 & 56.20 & 1.246 & 13.67 & 57.29 & 1.299 & 13.41\\
0.0186 & 57.90 & 1.408 & 8.25 & \textbf{58.44} & 1.438 & 6.67 & 56.10 & 1.259 & 9.54 & 56.19 & 1.306 & 9.41\\
0.0360 & 57.11 & 1.445 & 6.09 & \textbf{58.07} & 1.466 & 5.15 & 54.19 & 1.306 & 6.45 & 55.17 & 1.335 & 6.66\\
0.0695 & 55.87 & 1.482 & 4.36 & \textbf{57.29} & 1.508 & 4.08 & 50.21 & 1.388 & 4.17 & 53.34 & 1.397 & 4.76\\
0.1341 & 52.96 & 1.540 & 3.04 & \textbf{55.60} & 1.565 & 3.17 & 41.89 & 1.556 & 2.41 & 48.94 & 1.540 & 3.24\\
0.2590 & 46.22 & 1.645 & 1.93 & \textbf{52.12} & 1.649 & 2.35 & 26.71 & 1.809 & 1.09 & 34.72 & 1.789 & 1.59\\
0.5000 & 29.05 & 1.935 & 0.73 & \textbf{43.33} & 1.781 & 1.45 & 10.29 & 2.304 & 0.00 & 10.00 & 2.308 & 0.00\\
\midrule
\multicolumn{13}{l}{\textbf{STL-10}}\\
0.0050 & \textbf{56.73} & 1.376 & 15.04 & 56.44 & 1.441 & 11.88 & 54.75 & 1.277 & 19.20 & 54.31 & 1.322 & 19.37\\
0.0097 & 56.35 & 1.398 & 10.97 & \textbf{56.42} & 1.457 & 8.78 & 55.07 & 1.269 & 12.78 & 55.27 & 1.338 & 13.32\\
0.0186 & 55.83 & 1.427 & 8.03 & \textbf{56.22} & 1.473 & 6.65 & 53.20 & 1.304 & 8.33 & 54.21 & 1.351 & 9.28\\
0.0360 & 54.84 & 1.463 & 5.84 & \textbf{55.95} & 1.500 & 5.13 & 50.57 & 1.355 & 5.43 & 53.14 & 1.378 & 6.53\\
0.0695 & 53.28 & 1.505 & 4.14 & \textbf{55.33} & 1.535 & 4.05 & 47.06 & 1.426 & 3.54 & 51.82 & 1.426 & 4.67\\
0.1341 & 50.28 & 1.550 & 2.83 & \textbf{53.73} & 1.593 & 3.15 & 37.34 & 1.566 & 1.92 & 48.19 & 1.536 & 3.20\\
0.2590 & 44.34 & 1.641 & 1.81 & \textbf{50.67} & 1.661 & 2.34 & 29.50 & 1.696 & 1.19 & 36.37 & 1.709 & 1.62\\
0.5000 & 31.75 & 1.833 & 0.89 & \textbf{43.06} & 1.768 & 1.47 & 13.02 & 2.223 & 0.12 & 10.01 & 2.308 & 0.00\\
\midrule
\multicolumn{13}{l}{\textbf{Pooled test set (E2, E4 only)}}\\
0.0050 & -- & -- & -- & \textbf{71.13} & 1.003 & 11.59 & -- & -- & -- & 68.53 & 0.943 & 19.13\\
0.0097 & -- & -- & -- & \textbf{71.04} & 1.018 & 8.49 & -- & -- & -- & 69.02 & 0.953 & 13.02\\
0.0186 & -- & -- & -- & \textbf{70.86} & 1.036 & 6.43 & -- & -- & -- & 67.97 & 0.977 & 9.27\\
0.0360 & -- & -- & -- & \textbf{70.49} & 1.065 & 5.05 & -- & -- & -- & 66.72 & 1.017 & 6.63\\
0.0695 & -- & -- & -- & \textbf{69.76} & 1.103 & 4.06 & -- & -- & -- & 65.49 & 1.065 & 4.69\\
0.1341 & -- & -- & -- & \textbf{68.08} & 1.161 & 3.22 & -- & -- & -- & 60.72 & 1.212 & 3.34\\
0.2590 & -- & -- & -- & \textbf{64.22} & 1.265 & 2.45 & -- & -- & -- & 43.61 & 1.600 & 1.70\\
0.5000 & -- & -- & -- & \textbf{53.07} & 1.488 & 1.54 & -- & -- & -- & 14.78 & 2.288 & 0.00\\
\bottomrule\end{tabular}}
\end{table}

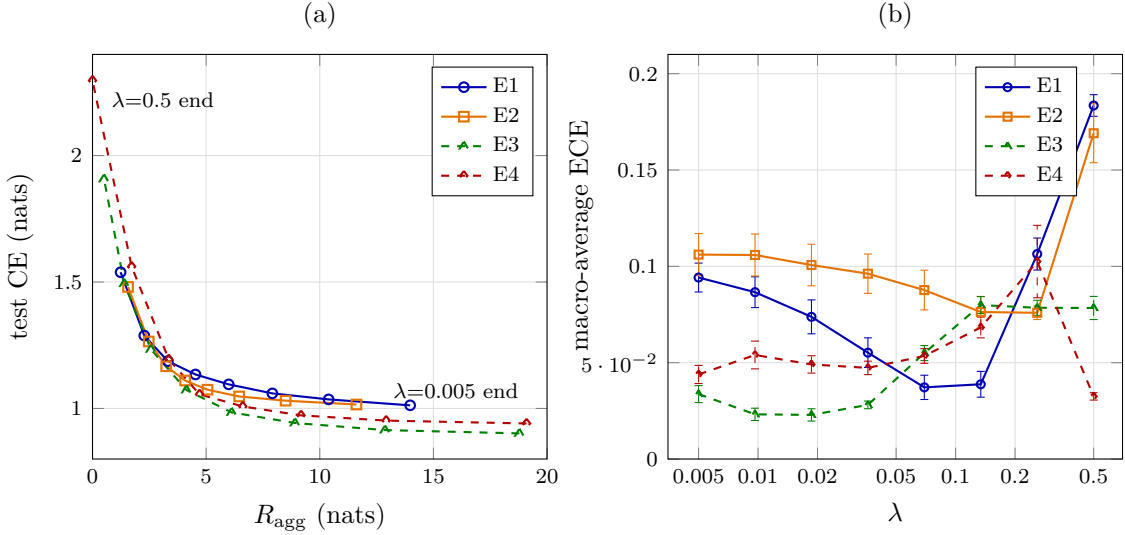
\begin{figure}[t]
\centering
\begin{tikzpicture}
\begin{groupplot}[
  group style={group size=2 by 1, horizontal sep=1.6cm},
  width=0.46\textwidth, height=0.42\textwidth,
  grid=major, grid style={gray!25},
  tick label style={font=\scriptsize}, label style={font=\small},
  title style={font=\small}, legend style={font=\scriptsize}]

% ---------- (a) CE vs surrogate rate ----------
\nextgroupplot[xlabel={$R_{\rm agg}$ (nats)}, ylabel={test CE (nats)},
  xmin=0, xmax=20, ymin=0.8, ymax=2.4, title={(a)},
  legend style={font=\scriptsize,at={(0.97,0.97)},anchor=north east}]
\addplot[blue!70!black,thick,mark=o,mark size=1.8pt] coordinates {(13.9766,1.0124) (10.3885,1.0358) (7.9123,1.0596) (5.9921,1.0952) (4.5295,1.1344) (3.3154,1.1873) (2.2776,1.2882) (1.2359,1.5387)};
\addlegendentry{E1}
\addplot[orange!90!black,thick,mark=square,mark size=1.8pt] coordinates {(11.6123,1.0156) (8.4956,1.0308) (6.4543,1.0474) (5.0611,1.0741) (4.0745,1.1104) (3.2426,1.1668) (2.4742,1.2649) (1.5635,1.4804)};
\addlegendentry{E2}
\addplot[green!50!black,thick,mark=triangle,mark size=1.8pt,dashed] coordinates {(18.7831,0.9013) (12.8687,0.9145) (8.9090,0.9425) (6.1219,0.9844) (4.1317,1.0730) (2.5673,1.2339) (1.3985,1.4935) (0.5050,1.9076)};
\addlegendentry{E3}
\addplot[red!70!black,thick,mark=diamond,mark size=1.8pt,dashed] coordinates {(19.0982,0.9406) (12.9221,0.9520) (9.1755,0.9726) (6.6116,1.0080) (4.7084,1.0559) (3.3565,1.1942) (1.7220,1.5609) (0.0001,2.2967)};
\addlegendentry{E4}
\node[font=\scriptsize,anchor=south west] at (axis cs:12.8,1.0) {$\lambda{=}0.005$ end};
\node[font=\scriptsize,anchor=west] at (axis cs:0.4,2.22) {$\lambda{=}0.5$ end};

% ---------- (b) ECE vs lambda ----------
\nextgroupplot[xmode=log, xmin=0.0035, xmax=0.7,
  xtick={0.005,0.01,0.02,0.05,0.1,0.2,0.5},
  xticklabels={0.005,0.01,0.02,0.05,0.1,0.2,0.5},
  xlabel={$\lambda$}, ylabel={macro-average ECE}, ylabel style={xshift=8pt,  yshift=-12pt}, ymin=0, ymax=0.21, title={(b)},
  legend style={font=\scriptsize,at={(0.9,0.97)},anchor=north east}]
\addplot[blue!70!black,thick,mark=o,mark size=1.5pt,error bars/.cd,y dir=both,y explicit] coordinates {(0.005,0.0942) +- (0,0.0075) (0.009653,0.0866) +- (0,0.0080) (0.018638,0.0738) +- (0,0.0088) (0.035984,0.0552) +- (0,0.0077) (0.069475,0.0372) +- (0,0.0063) (0.134135,0.0388) +- (0,0.0067) (0.258974,0.1064) +- (0,0.0083) (0.5,0.1835) +- (0,0.0056)};
\addlegendentry{E1}
\addplot[orange!90!black,thick,mark=square,mark size=1.5pt,error bars/.cd,y dir=both,y explicit] coordinates {(0.005,0.1061) +- (0,0.0109) (0.009653,0.1059) +- (0,0.0109) (0.018638,0.1007) +- (0,0.0108) (0.035984,0.0962) +- (0,0.0102) (0.069475,0.0877) +- (0,0.0103) (0.134135,0.0763) +- (0,0.0081) (0.258974,0.0759) +- (0,0.0035) (0.5,0.1691) +- (0,0.0153)};
\addlegendentry{E2}
\addplot[green!50!black,thick,mark=triangle,mark size=1.5pt,dashed,error bars/.cd,y dir=both,y explicit] coordinates {(0.005,0.0337) +- (0,0.0044) (0.009653,0.0232) +- (0,0.0032) (0.018638,0.0229) +- (0,0.0032) (0.035984,0.0281) +- (0,0.0020) (0.069475,0.0550) +- (0,0.0039) (0.134135,0.0799) +- (0,0.0044) (0.258974,0.0785) +- (0,0.0039) (0.5,0.0784) +- (0,0.0060)};
\addlegendentry{E3}
\addplot[red!70!black,thick,mark=diamond,mark size=1.5pt,dashed,error bars/.cd,y dir=both,y explicit] coordinates {(0.005,0.0439) +- (0,0.0047) (0.009653,0.0540) +- (0,0.0072) (0.018638,0.0491) +- (0,0.0045) (0.035984,0.0473) +- (0,0.0035) (0.069475,0.0535) +- (0,0.0038) (0.134135,0.0683) +- (0,0.0054) (0.258974,0.1025) +- (0,0.0188) (0.5,0.0325) +- (0,0.0019)};
\addlegendentry{E4}
\end{groupplot}
\end{tikzpicture}
\caption{Macro-average over 10 seeds (E1 blue, E2 orange, E3 green dashed, E4 red dashed). (a) Test CE versus surrogate rate $R_{\rm agg}$, from $\lambda=0.005$ (largest $R_{\rm agg}$) to $0.5$. (b) ECE versus $\lambda$ (bars: 95\% CI half-widths). The low E4 ECE at $\lambda=0.5$ reflects a collapsed predictor.}
\label{fig:ce-rate-ece}
\end{figure}

\begin{table}[t]\centering\small
\caption{Empirical proxy for the sufficiency deficit $\epsilon_U=I(X;Y\mid U)$ of the frozen
interface $U=\mathrm{softmax}(\ell_\vartheta(F_\theta(X)))$, on the full test sets
(mean over 10 seeds). $\mathrm{CE}_H$ is the test cross-entropy of the frozen $F$ with
its stored linear head, an upper bound on $H(Y\mid U)$; $\mathrm{CE}_X$ is that of a wider
end-to-end CNN trained on $X$, an upper bound on $H(Y\mid X)$.
$\hat\epsilon_H=\mathrm{CE}_H-\mathrm{CE}_X$ is a difference of two upper bounds, and small negative values (CIFAR-10) are
estimation error, i.e.\ indistinguishable from zero. $\Phi(\hat\epsilon)$ is the uniform
stability bound of Theorem~\ref{thm:stab} evaluated at $\hat\epsilon$ (negative values
clipped to $0$).}\label{tab:suff}
\begin{tabular}{@{}lcccccc@{}}\toprule
 & \multicolumn{3}{c}{Per-domain $F$ (E1)} & \multicolumn{3}{c}{Pooled $F$ (E2)}\\
\cmidrule(lr){2-4}\cmidrule(lr){5-7}
Dataset & $\mathrm{CE}_H$ & $\hat\epsilon_H$ & $\Phi(\hat\epsilon_H)$ & $\mathrm{CE}_H$ & $\hat\epsilon_H$ & $\Phi(\hat\epsilon_H)$\\\midrule
MNIST   & 0.116 & 0.009 & 0.39 & 0.146 & 0.039 & 0.71\\
SVHN    & 0.917 & 0.025 & 0.59 & 0.909 & 0.018 & 0.52\\
CIFAR10 & 1.189 & $-0.001$ & 0.00 & 1.238 & 0.049 & 0.78\\
STL10   & 1.223 & 0.011 & 0.43 & 1.296 & 0.084 & 0.96\\
\bottomrule\end{tabular}
\end{table}

%% ---------------------------------------------------------------------------
%% ---------------------------------------------------------------------------
\subsection{What the rate controller must express: noisy XOR}\label{sec:xor}
Sufficiency says separation is possible, not that a given compressor class realizes it.

\begin{example}[Noisy XOR]\label{ex:xor}
Let $X$ be uniform on $\{\pm1\}^2$, $Y_0=X_1X_2$, and $Y=Y_0(-1)^{E}$ with
$E\sim\mathrm{Bern}(p)$ independent of $X$, $p=0.1$. Since $Y-Y_0-X$, the map
$F(X)=Y_0$ is sufficient ($\epsF=0$). Let $U=Y_0$ (uniform on $\{\pm1\}$) and compare
(a) \emph{Gaussian-noise compressors} $T=\alpha U+\xi$, $\xi\sim\mathcal N(0,1)$
independent of $U$, with $\alpha\ge0$ free for each $\lambda$; and
(b) unrestricted channels $P_{T|U}$.
\end{example}

\begin{proposition}[Unrestricted optimum; Gaussian-noise gap]\label{prop:mgl}
In Example~\ref{ex:xor}, let $g_\lambda(q)=\hb(p*q)+\lambda(\log2-\hb(q))$. Then:
\begin{enumerate}[label=(\alph*),leftmargin=*,itemsep=1pt]
\item $\Js(\lambda)=\min_{q\in[0,1/2]}g_\lambda(q)$, attained by
$T=Y_0(-1)^N$, $N\sim\mathrm{Bern}(q)$ independent of $Y_0$.
\item For the unrestricted class (b), $\Gap_{F,\cC}(\lambda)=0$ for every $\lambda>0$.
\item If $0<\lambda<(1-2p)^2$, every minimizer satisfies $q^\star\in(0,\tfrac12)$, and the
Gaussian-noise class (a) has a strictly positive gap,
$\inf_{\alpha\ge0}J_\lambda(\alpha)-\Js(\lambda)>0$, where $J_\lambda(\alpha)$ is the
objective of $T=\alpha Y_0+\xi$. The gap is positive for each fixed $\lambda$, not
uniformly in $\lambda$.
\end{enumerate}
\end{proposition}

\begin{proof}
\emph{Step 1 (formula for $\Js$).} Since $Y-Y_0-X$, $F=Y_0$ is sufficient ($\epsF=0$), so by
Theorem~\ref{thm:exact} it suffices to consider channels $P_{T|U}$ with $U=Y_0$; then
$\I(X;T)=\I(Y_0;T)=:\iota$ and $H(Y_0)=\log2$. 
Define $p*q=p(1-q)+q(1-p)$ and
$\phi_p(u)=\hb(p*\hb^{-1}(u))$, where $\hb^{-1}$ denotes the
inverse of binary entropy on $[0,\tfrac12]$. For $p\in(0,\tfrac12)$,
$\phi_p$ is strictly convex on $(0,\log2)$ \cite{wyner73}.
Given $T=t$, $Y_0$ is Bernoulli and $Y$ is its
image through a $\mathrm{BSC}(p)$, so $H(Y|T=t)=\phi_p\big(H(Y_0|T=t)\big)$; this is an equality, and the convexity of $\phi_p$ is the content of Mrs.\ Gerber's lemma \cite{wyner73}.
So, $H(Y|T)\ge\phi_p(\log2-\iota)$, hence
$\Jl\ge\phi_p(\log2-\iota)+\lambda\iota$. A $\mathrm{BSC}(q)$ from $Y_0$ has
$\iota=\log2-\hb(q)$ and $\phi_p(\hb(q))=\hb(p*q)$, so it attains equality. This proves (a).

\emph{Step 2 (interior minimizer).} We have
$g_\lambda'(q)=(1-2p)\hb'(p*q)-\lambda\hb'(q)$ with $\hb'(q)=\log\frac{1-q}{q}$. As
$q\downarrow0$, $g_\lambda'(q)\to-\infty$, so $q=0$ is not a minimizer. At $q=\tfrac12$,
$g_\lambda'=0$ and $g_\lambda''(\tfrac12)=-4\big[(1-2p)^2-\lambda\big]<0$ for
$\lambda<(1-2p)^2$, so $q=\tfrac12$ is a local maximum. By compactness, every minimizer
$q^\star$ lies in $(0,\tfrac12)$ and solves $g_\lambda'(q^\star)=0$, i.e.
\(
(1-2p)\log\frac{1-p*q^\star}{p*q^\star}=\lambda\log\frac{1-q^\star}{q^\star}.
\)

\emph{Step 3 (unrestricted class).} The channel of Step~1 with $q=q^\star(\lambda)$ is a
channel on $U=Y_0$, so it belongs to the unrestricted class and attains $\JF(\lambda)=\Js(\lambda)$.
Hence $J^{F,\cC}(\lambda)=\JF(\lambda)=\Js(\lambda)$ and $\Gap_{F,\cC}(\lambda)=0$, which is (b).
The budget enters only through the scalar $q^\star(\lambda)$, as in the Gaussian case, but the
optimal noise is binary rather than Gaussian. Equivalently, in the decomposition
\eqref{eq:decomp} both terms vanish for class (b), whereas for class (a) the representation
term $\Gap_F$ is still zero and the whole gap comes from compressor expressiveness,
$J^{F,\cC}-\JF$.

\emph{Step 4 (Gaussian-noise class).} For $T=\alpha Y_0+\xi$, $\Pr(Y_0{=}1|t)=\mathrm{sig}(2\alpha t)$, \(\mathrm{sig}\) is sigmoid function, so $H(Y_0|T=t)$ is nonconstant for $\alpha>0$ and Jensen is strict:
$J_\lambda(\alpha)>\Js(\lambda)$ for all $\alpha\in(0,\infty)$. $J_\lambda$ extends
continuously to $[0,\infty]$ with $J_\lambda(0)=\log2=g_\lambda(\tfrac12)$ and
$J_\lambda(\infty)=g_\lambda(0)$, both strictly larger than $\Js(\lambda)$ by Step~2. A
continuous function on the compact $[0,\infty]$ that exceeds $\Js(\lambda)$ everywhere has a
minimum exceeding $\Js(\lambda)$, which proves (c).
\end{proof}

\section{Learning algorithms and cost model}\label{sec:alg}
Fig.~\ref{fig:stages} shows the pipeline as implemented in the four-dataset experiments.
\paragraph{Separated pipeline.}
The compressor receives $U:=Z=\mathrm{softmax}(\ell_\vartheta(V))$ with $V=F_\theta(X)$; $V$ is used only to compute it. The representation $F$ of the separation theory is therefore $\mathrm{softmax}\circ\ell_\vartheta\circ F_\theta$, not $F_\theta$ alone.
\emph{Stage~1} trains $F_\theta$ and a trainable affine classifier head $\ell_\vartheta(v)=B_\ell v+c_\ell$, $\vartheta=(B_\ell,c_\ell)$, by cross-entropy only (no rate term, no $\lambda$):
\begin{equation}
(\theta^\star,\vartheta^\star)=\argmin_{\theta,\vartheta}\E\big[-\log\mathrm{softmax}(\ell_\vartheta(F_\theta(X)))_Y\big].
\label{eq:stage1}
\end{equation}
$F_{\theta^\star}$ and $\ell_{\vartheta^\star}$ are then frozen. In the four-dataset implementation the head is \emph{retained}, because the implemented interface is $Z=\mathrm{softmax}(\ell_{\vartheta^\star}(F_{\theta^\star}(x)))\in\R^{10}$; the head's test cross-entropy equals $\mathrm{CE}(\psi_0\circ F_{\rm int})$ for the unsmoothed predictor $\psi_0(z)=z$; the smoothed predictor $\psi_\gamma$ of Section~\ref{sec:sens} is used only to state the bounded-log-loss certificate and is not evaluated in the experiments. 

\emph{Stage~2} trains, for each fixed budget $\lambda_j$, one compressor $C_j$ and decoder $G_j$ on cached interface features, minimizing $\mathrm{CE}+\lambda_jR_{\rm agg}$ (Lemma~\ref{lem:var}). The latent is $T=\mu+\sigma\odot\zeta$, $\zeta\sim\mathcal N(0,\Id)$, with $(\mu,\log\sigma^2)=C_j(Z)$, and the aggregate prior is a diagonal Gaussian $\mathcal N(\bar\mu,\sigma_T^2)$,
\begin{equation}
\bar\mu=\mathrm{mean}_b(\mu),\quad \sigma_T^2=\mathrm{var}_b(\mu)+\mathrm{mean}_b(\sigma^2),\quad R_{\rm agg}=\tfrac1{|b|}\textstyle\sum_{x\in b}\KL\big(\mathcal N(\mu_x,\sigma_x^2)\,\|\,\mathcal N(\bar\mu,\sigma_T^2)\big),
\label{eq:Ragg}
\end{equation}
computed per minibatch $b$ with stop-gradient through $(\bar\mu,\sigma_T^2)$ during training. At evaluation, where the prior statistics are fitted on training inputs and frozen, Lemma~\ref{lem:var} with $\nu=\mathcal N(\bar\mu,\sigma_T^2)$ gives $R_{\rm agg}\ge\I(X;T)$; during training the minibatch statistics make $R_{\rm agg}$ only a surrogate, whose tightness differs across models.

\paragraph{Why two-stages-learning?}
Let $U=F_{\rm int}(X)$ be the probability interface of
Section~\ref{sec:sens}, so $Y-X-U-T$ and $\epsilon_F=I(X;Y\mid U)$.
Since $I(X;T)=I(U;T)$ and
$H(Y|T)-H(Y|X)=\epsilon_F+I(U;Y|T)$,
\(
\Jl(C\circ F)=H(Y\mid X)+\epsilon_F+I(U;Y\mid T)+\lambda I(U;T).
\)
Thus $\epsilon_F$ depends on neither $C$ nor $\lambda$, and
$\JF(\lambda)=H(Y|U)+\Psi_F(\lambda)$ with
$\Psi_F(\lambda)=\inf_{P_{T|U}}\{I(U;Y|T)+\lambda I(U;T)\}$, which
depends on $F$ only through $P_{UY}$. This is why Stage~1 involves no
$\lambda$ and Stage~2 is repeated per budget. Cross-entropy training controls $\epsilon_F$, which depends on $F$ through $P_{UY}$.

\emph{Stage~1.} By Eq.~\eqref{eq:sens-defs}, $\mathrm{CE}(\psi\circ F)=H(Y|X)+\epsilon_F+\eta_\psi(F)$, so
$\epsilon_F\le\mathrm{CE}(\psi\circ F)-H(Y|X)$ at the population level
(equality iff $\psi=\pi_F$). 
\emph{Stage~2.} For a decoder $g$ and a prior $\nu$ fixed in advance,
Lemma~\ref{lem:var} and Theorem~\ref{thm:dec-id} give
$\mathrm{CE}(g)+\lambda R_\nu=\Jl+\delta_G+\lambda\,\KL(P_T\|\nu)$ with
$\delta_G=\E\,\KL(P_{Y|T}\|g(\cdot|T))$. 
Theorems~\ref{thm:exact} and~\ref{thm:stab} give $\JF=\Js$ if
$\epsilon_F=0$ and $\JF-\Js\le\Phi(\epsilon_F)$ in general, for
unrestricted compressors.

\section{Experiments}\label{sec:exp}
\subsection{Design and protocol}\label{sec:multi}

\paragraph{Questions addressed.}
The experiments address the following questions: (i) whether a frozen representation with one compressor--decoder pair per budget matches independently trained VIBs in accuracy (Table~\ref{tab:main}, Fig.~\ref{fig:results}a, Tables~\ref{tab:app1}--\ref{tab:app2}); (ii) what this does to cross-entropy, surrogate rate and calibration (Fig.~\ref{fig:results}b,c, Fig.~\ref{fig:ce-rate-ece}); (iii) what it saves (Table~\ref{tab:architecture_cost}); and (iv) whether the frozen interface looks nearly sufficient on a coarse proxy (Table~\ref{tab:suff}). Fig.~\ref{fig:stages} and Algorithm~\ref{alg:sep} specify the separated pipeline evaluated in E1 and E2.
Four main experiments cross \{separated pipeline, end-to-end VIB\} with \{per-domain, pooled\}; in each, one model is trained \emph{per budget} and no model receives $\lambda$ or a domain identifier as input.
\begin{itemize}[leftmargin=*,itemsep=1pt]
\item \textbf{E1 (per-domain separation).} For each domain $i=\{1,2,3,4\}$, train a representation $F_i$ once, freeze it, and train a pair $(C_{ij},G_{ij})$ for each budget $\lambda_j$, Algorithm~\ref{alg:sep} for \(J=8\): 4 representations, 32 pairs. Compared with E3.
\item \textbf{E2 (pooled separation).} Shuffle the four training sets, each from one domain, into one pool, train one shared $F$, freeze it, and train one pair $(C_j,G_j)$ per budget, Algorithm~\ref{alg:sep} for \(J=8\): 1 representation, 8 pairs. Compared with E4.
\item \textbf{E3 (per-domain VIB).} For each domain $i=\{1,2,3,4\}$ and budget $\lambda_j, j=\{1,2,\dots,8\}$, independently train an end-to-end VIB per each budget: 32 models.
\item \textbf{E4 (pooled VIB).} For each budget $\lambda_j, j=\{1,2,\dots,8\}$, independently train one VIB on the shuffled pool: 8 models.
\end{itemize}
E1 versus E3 compares the two learning strategies within each domain; E2 versus E4 compares them in the pooled setting. E2 versus E1 and E4 versus E3 are \emph{not} clean comparisons of pooling: the pooled experiments use four times as many training images, and the common label space is not semantically consistent (below).

\paragraph{Additional baselines (E5--E8).}
Two further families, our own implementations, are run on the same data, grid, seeds and pool, and appear in Table~\ref{tab:main}, Fig.~\ref{fig:results} and Table~\ref{tab:architecture_cost}. \textbf{E5/E6} are FiLM-conditioned VIBs: one network that \emph{does} receive $\lambda$ as input (E6: one per domain; E5: one on the pool, numbering as in Table~\ref{tab:main}), step-matched to the E3 training budget. \textbf{E7/E8} are OIB (bank-of-NN): a pretrained MLP on DFT features, a closed-form Gaussian-IB compressor and one decoder per budget (E7 per domain, E8 pooled). These rows are descriptive: paired intervals are computed for E1$-$E3 and E2$-$E4.

\paragraph{Data and training.}
MNIST, SVHN, CIFAR-10 and STL-10 (resized to $32\times32$, scaled to $[-1,1]$); per seed, $5000$
training images per dataset ($10\%$ for validation), full official test sets (pooled: $54032$
images). Labels are matched by index (digit $i$, CIFAR/STL class \texttt{CIFAR\_ORDER}$[i]$,
\texttt{STL\_ORDER}$[i]$). The grid is eight log-spaced budgets
$\lambda\in\{0.0050,0.0097,0.0186,0.0360,0.0695,0.1341,0.2590,0.5000\}$ (so $\beta=1/\lambda\in\{200,\dots,2\}$), shared by all experiments and not tuned.
$F$ is a three-convolution CNN with output dimension $m=32$ and an affine head; $C_j$ and
$G_j$ are small MLPs (Section~\ref{sec:alg}). All stages use Adam for $50$ epochs and the
best-validation-epoch weights (the pooled representation overfits after about $15$ epochs, so its checkpoint selection matters); the aggregate prior is fitted on up to $10^4$ training inputs and
frozen. Metrics use $5$ noise draws, and ECE uses $15$ bins. We report means over $10$ seeds with
$t$-based $95\%$ intervals and paired differences averaged over the grid.

\subsection{Results}\label{sec:results}

\paragraph{Accuracy.}
The separated pipelines have higher grid-averaged accuracy (differences in percentage points, pp): $+8.29$ pp [$+7.75,+8.83$] for
E1$-$E3 and $+9.77$ pp [$+8.59,+10.95$] for E2$-$E4 (macro), and $+10.08$ pp [$+8.81,+11.35$]
for E2$-$E4 on the pooled test set. The advantage is small at small $\lambda$ ($1.60$ pp at
$\lambda=0.005$ for E1$-$E3, $1.72$ pp for E2$-$E4) and large at large $\lambda$, where the VIBs are near to
collapse: at $\lambda=0.5$ E4 reaches $12.63\%$ (CE $2.2970\approx\ln10$), against $47.45\%$ (E1)
and $51.91\%$ (E2) for the separated pipelines (Fig.~\ref{fig:results}a).
Table~\ref{tab:main} gives the budget-averaged levels: E1 $65.59\pm0.49$ versus E3 $57.30\pm0.23$ per domain, and E2 $66.64\pm0.75$ versus E4 $56.87\pm0.91$ (macro; $67.14\pm0.90$ versus $57.06\pm0.90$ on the pooled test set).
The ordering is not driven by one domain. In Tables~\ref{tab:app1} and~\ref{tab:app2}, the mean accuracy of E1 exceeds that of E3, and that of E2 exceeds that of E4, in all $32$ dataset--budget cells (E2 also exceeds E4 at all eight budgets on the pooled test set); at $\lambda=0.005$ the E1$-$E3 margin ranges from $0.3$ pp (MNIST) to $4.0$ pp (SVHN).
The collapse at $\lambda=0.5$ is visible per dataset: E4 sits at chance ($10.00\%$ on CIFAR-10, $10.01\%$ on STL-10) with $R=0.00$, and E3 has $R\le0.12$ nats on three of four datasets, whereas the separated pipelines keep $R\ge0.73$ nats and reach $29$--$76\%$ (E1) on every dataset. The wide interval of E4 at $\lambda=0.259$ ($\pm7.06$ pp, Fig.~\ref{fig:results}a) shows that the onset of collapse is itself seed dependent.

\paragraph{Rate and objective.}
At small $\lambda$ the separated pipelines use less surrogate rate ($13.64$ versus $19.00$ nats for
E1/E3, $11.70$ versus $19.23$ for E2/E4 at $\lambda=0.005$) but have higher cross-entropy ($+0.1228$
and $+0.0724$ nats). The net effect on $\overline{\mathcal L}$ is mixed: $\overline{\mathcal L}_{\rm E1}-\overline{\mathcal L}_{\rm E3}=+0.0735$ [$+0.0658,+0.0812$] for
E1$-$E3; for E2$-$E4 the macro difference is not significant ($-0.0144$ [$-0.0331,+0.0043$]), while on the pooled test
set E2 is lower ($-0.0351$ [$-0.0534,-0.0168$]).
The small-$\lambda$ penalty is offset on average by the VIB's loss at large $\lambda$ (Fig.~\ref{fig:results}b): the VIBs have lower CE for $\lambda\le0.0695$ and higher CE from $\lambda=0.1341$ on, so the grid-averaged CE is nearly equal for E1/E3 ($1.1709$ versus $1.1806$) and lower for E2 than E4 ($1.1461$ versus $1.2472$, Table~\ref{tab:main}). The surrogate rate decreases monotonically in $\lambda$ for every method (Fig.~\ref{fig:results}c), sweeping E1 from $13.6$ to $1.3$ nats and E4 to $0.0002$ nats at $\lambda=0.5$, so equal nominal $\lambda$ does not mean equal rate.
Fig.~\ref{fig:ce-rate-ece}a therefore plots CE against the measured $R_{\rm agg}$. Reading the plotted means, E1 and E3 nearly coincide below about $2.5$ nats, and E3 has lower CE for $R\gtrsim3$ nats (by $0.03$--$0.10$ nats up to $R\approx6$); E2 lies below E4 for $R\lesssim4$ nats (by up to $\approx0.13$) and above it for larger $R$ (by $\lesssim0.04$). The accuracy advantage at the nominal budgets thus reflects that the separated pipeline keeps operating at useful rates at all eight $\lambda$, also at higher rates, the end-to-end VIBs are the better compressors.

\paragraph{Calibration.}
Fig.~\ref{fig:ce-rate-ece}b and Table~\ref{tab:main} show (expected calibration error, ECE \cite{guo17}, 15 bins) that the separated pipelines are less calibrated on average (ECE $0.0862$ versus $0.0511$ for E1/E3; $0.1003$ versus $0.0569$ for E2/E4; $0.0804$ versus $0.0357$ on the pooled test set). The gap is largest at small $\lambda$ ($\approx0.09$--$0.11$ versus $0.03$--$0.04$ at $\lambda=0.005$), and E1 is non-monotone in $\lambda$ (minimum $0.0372$ at $\lambda=0.0695$, $0.1835$ at $\lambda=0.5$). 

\paragraph{FiLM and OIB baselines.}
We compare the $95\%$ intervals of Table~\ref{tab:main}. The FiLM-conditioned VIB is the strongest baseline: it has the lowest macro CE among per-domain methods ($1.1227$ for E6 versus $1.1709$ for E1 and $1.1806$ for E3) and, jointly with E3, the lowest $\overline{\mathcal L}$ ($1.4126\pm0.0121$ versus $1.4139\pm0.0043$ for E3 and $1.4874\pm0.0064$ for E1). It also degrades more gracefully than independent VIBs at $\lambda=0.5$ ($45.15\%$ and $37.40\%$ for E6 and E5 versus $27.71\%$ and $12.63\%$ for E3 and E4, Fig.~\ref{fig:results}a). Its budget-averaged accuracy is nevertheless $60.86\%$ (E6) and $60.67\%$ (E5), i.e.\ $4.7$ and $6.0$ pp below E1 and E2. In the pooled setting the macro $\overline{\mathcal L}$ intervals of E5 and E2 overlap ($1.4840\pm0.0120$ versus $1.4646\pm0.0175$), while on the pooled test set E2 is lower ($1.4540\pm0.0168$ versus $1.5007\pm0.0148$).
OIB (E7/E8) reaches only $42.72\%$ and $38.98\%$, and its accuracy is nearly flat across a $100\times$ range of $\lambda$ ($43.4\%$ to $39.0\%$ for E7), consistent with a closed-form compressor on fixed features whose rate $R_{\rm agg}$ ($128$--$149$ nats, Table~\ref{tab:main}) is not commensurate with the learned compressors; it is therefore omitted from Fig.~\ref{fig:results}c.

\paragraph{Cost.}
Separation reduces estimated training MACs by $7.7\times$ ($3.039\times10^{13}$ versus
$2.347\times10^{14}$) and wall-clock by $2.1\times$ (E1 vs.\ E3, $141.3$ versus $292.9$ s) and
$2.1\times$ (E2 vs.\ E4, $138.2$ versus $294.1$ s). Wall-clock is dominated by the many small $(C,G)$ fits, so MACs overstate
the realized speed-up. MAC counts are approximate estimates. A new budget costs one additional $(C,G)$ fit, not zero.
Table~\ref{tab:architecture_cost} gives the full-grid totals for all methods. The step-matched FiLM networks (E5/E6)use $2.355\times10^{14}$ MACs, close to E3 ($2.347\times10^{14}$) by construction, and about $3.5\times$ the wall-clock of the separated pipeline ($487.3/489.3$ s versus $141.3/138.2$ s). OIB is the cheapest ($1.83/1.75\times10^{12}$ MACs, $76.4/73.9$ s, i.e.\ $\approx17\times$ fewer MACs and $\approx1.9\times$ less time than separation) but is $23$--$28$ pp less accurate (Table~\ref{tab:main}).

\paragraph{Frozen representation and sufficiency proxy.}
Table~\ref{tab:suff} is the link between the experiments and Theorem~\ref{thm:stab}. It reports $\hat\epsilon=\mathrm{CE}_{\rm head}-\mathrm{CE}_X$ for the frozen interface. For the per-domain $F$ (E1), $\hat\epsilon\le0.025$ nats on every dataset (MNIST $0.009$, SVHN $0.025$, CIFAR-10 $-0.001$, STL-10 $0.011$), small next to a CE of $0.9$--$1.2$ on SVHN, CIFAR-10 and STL-10. For the pooled $F$ (E2) it is larger on three of four datasets ($0.039$, $0.049$, $0.084$ for MNIST, CIFAR-10, STL-10), as expected if one $F$ must serve four unaligned label semantics (the alignment term of Proposition~\ref{prop:pooled}). This is directionally consistent with E2 having higher CE than E1 on CIFAR-10 and STL-10 at small $\lambda$ (Table~\ref{tab:app2}), but not on MNIST (Table~\ref{tab:app1}).
Two caveats apply. Since $\hat\epsilon=\epsilon_U+\eta_{\rm head}-\eta_X$ is a difference of two upper bounds, its error can have either sign, so it neither bounds nor lower-bounds $\epsilon_U$. And $\Phi(\hat\epsilon)$ ($0$--$0.96$ nats) is of the order of the CE itself.

\paragraph{Summary of evidence.}
Higher budget-averaged accuracy than independently trained VIBs in every dataset--budget cell, no collapse over the grid, and $7.7\times$ fewer training MACs ($\approx2\times$ wall-clock). Against FiLM, separation is more accurate ($4.7$--$6.0$ pp), uses about the same MACs and about $3.5\times$ less wall-clock, but the per-domain FiLM has lower CE and $\overline{\mathcal L}$.

Tables~\ref{tab:app1} and~\ref{tab:app2} list accuracy, test cross-entropy and surrogate rate for each dataset (and the pooled test set) at each of the eight prespecified budgets; Fig.~\ref{fig:ce-rate-ece} shows CE against surrogate rate (a) and the macro-average ECE (b). All entries are means over 10 seeds. The grid is shared by all experiments; the tables are not used to pick an operating point. Rows at the same $\lambda$ are equal-weight, not equal-rate, comparisons.

%% ---------------------------------------------------------------------------
%% ---------------------------------------------------------------------------
\section{Conclusion}\label{sec:disc}

Gaussian IB shows that information-bearing coordinates and rate allocation are different
objects. Throughout, the budget is given and not optimized. We generalized this to a statement about \emph{which representations carry the whole IB
path}: exactly the sufficient ones, with
worst-case loss bounded by the single deficit $\epsF$ through $\Phi$, by
$\sum_sw_s\epsF(P_s)+I(Y;S|F(X))$ for pooled domains, and with no
loss at all for $\lambda\ge1$. The same separation holds at the level
of the IB self-consistent iteration.

Experimentally, freezing a CE-trained representation and fitting one head per budget gives
\textbf{a $7.7\times$ reduction in estimated training MACs} ($2.1\times$ in wall-clock) and
\textbf{higher budget-averaged accuracy than VIBs trained independently at each $\lambda$} ($\mathbf{+8.29}$ pp per domain,
$\mathbf{+10.08}$ pp on the pooled test set), with the largest accuracy gains at large $\lambda$, where the
VIBs collapse. At small $\lambda$ the separated pipeline has higher cross-entropy and ECE, and its surrogate
Lagrangian is higher in the per-domain setting ($+0.0735$ nats, 95\% CI $[+0.0658,+0.0812]$,
macro-average), not distinguishable from that of the pooled VIB in the macro average ($-0.0144$ nats,
$[-0.0331,+0.0043]$) and lower on the pooled test set ($-0.0351$ nats, $[-0.0534,-0.0168]$). The pooled representation is checkpoint-sensitive. Overall, separation gives lower training cost and higher budget-averaged accuracy against VIBs trained independently at each $\lambda$, at the price of cross-entropy and calibration at small $\lambda$.

%% ---------------------------------------------------------------------------
%% ---------------------------------------------------------------------------

\end{document}